\documentclass{article}

\usepackage{microtype}
\usepackage{graphicx}
\usepackage{subfigure}
\usepackage{booktabs} 
\usepackage{amssymb}
\usepackage{amsmath}
\usepackage{amsthm}

\usepackage{hyperref}

\usepackage[ruled]{algorithm2e}

\usepackage{amsmath}
\usepackage{amssymb}
\usepackage{mathtools}
\usepackage{amsthm}
\usepackage{mathtools}

\usepackage[utf8]{inputenc}

\usepackage{cancel}
\usepackage{mathtools}
\usepackage{thm-restate}
\usepackage[mathscr]{euscript}
\usepackage[toc,page,header]{appendix}
\usepackage{dsfont}

\usepackage{authblk}
\usepackage{amssymb}
\usepackage{amsmath}
\usepackage{amsthm}
\usepackage{latexsym}
\usepackage{graphicx}
\usepackage{color}
\usepackage{graphics}

\usepackage[capitalize,noabbrev]{cleveref}

\usepackage{MnSymbol}
\DeclareMathAlphabet\mathbb{U}{msb}{m}{n}
\usepackage{xpatch}

\def\Rset{\mathbb{R}}
\def\Nset{\mathbb{N}}

\let\P\undefined

\DeclareMathOperator*{\P}{\mathbb{P}}

\DeclarePairedDelimiter{\ceil}{\lceil}{\rceil}

\newcommand{\cF}{\mathcal{F}}

\newcommand{\cT}{\mathcal{T}}

\newcommand{\sfL}{{\mathsf L}}
\newcommand{\sfM}{{\mathsf M}}
\newcommand{\sfS}{{\mathsf S}}

\DeclareMathOperator{\SL}{\mathsf{F}}
\DeclareMathOperator{\sfoE}{\mathsf{E}}

\newcommand{\ignore}[1]{}

\newcommand{\1}{\mathds{1}}

\definecolor{darkred}{RGB}{140,0,0}
\definecolor{darkgreen}{RGB}{0,140,0}

\usepackage{graphicx}
\usepackage{color,soul}
\definecolor{lightred}{rgb}{1,0.78,0.78}
\sethlcolor{lightred}

\newcommand{\sX}{{\mathscr X}}
\newcommand{\sY}{{\mathscr Y}}
\usepackage{wrapfig}

\theoremstyle{plain}

\theoremstyle{definition}

\theoremstyle{remark}

\usepackage[textsize=tiny]{todonotes}

\newcommand{\sD}{{\mathscr D}}

\newcommand{\sN}{{\mathscr N}}

\newcommand{\sT}{{\mathscr T}}

\usepackage[utf8]{inputenc} 
\usepackage[T1]{fontenc}    
\usepackage{hyperref}       
\usepackage{url}            
\usepackage{booktabs}       
\usepackage{amsfonts}       
\usepackage{nicefrac}       
\usepackage{microtype}      
\usepackage{xcolor}         
\usepackage{graphicx}
\usepackage{amsmath}
\usepackage{dsfont}
\usepackage{cleveref}
\usepackage{soul}
\usepackage{hyphenat}
\usepackage{natbib}
\usepackage{geometry}
\definecolor{lightred}{rgb}{1,0.78,0.78}
\sethlcolor{lightred}

\title{Language Models with Conformal Factuality Guarantees}
\author{Christopher Mohri}
\author{Tatsunori Hashimoto}
\affil{Department of Computer Science, Stanford University}
\date{}
\begin{document}

\maketitle

\begin{abstract}
\noindent Guaranteeing the correctness and factuality of language model (LM) outputs is a major open problem. In this work, we propose \emph{conformal factuality}, a framework that can ensure high probability correctness guarantees for LMs by connecting language modeling and conformal prediction.
We observe that the correctness of an LM output is equivalent to an uncertainty quantification problem, where the uncertainty sets are defined as the entailment set of an LM's output.
Using this connection, we show that conformal prediction in language models corresponds to a back-off algorithm that provides high probability correctness guarantees by progressively making LM outputs less specific (and expanding the associated uncertainty sets). This approach applies to any black-box LM and requires very few human-annotated samples. Evaluations of our approach on closed book QA (FActScore, NaturalQuestions) and reasoning tasks (MATH) show that our approach can provide 80-90\% correctness guarantees while retaining the majority of the LM's original output.
\end{abstract}

\section{Introduction}

Large language models (LLMs) have demonstrated exceptional progress in recent years and are increasingly being adopted in various domains such as search engines and chatbots \citep{wei2022emergent, raffel2023exploring, bubeck2023sparks, ling2023domain}. However, their outputs cannot be fully trusted due to their tendency to generate hallucinations and non-factual content \citep{maynez2020faithfulness, huang2023survey, Ji_2023}. This has made the factuality and correctness of LLMs an important and active area of research, with several promising approaches that ground LLMs with knowledge sources \citep{wang2023survey, lee2023factuality, semnani-etal-2023-wikichat, lewis2021retrievalaugmented, du2023improving, he2022rethinking} or perform abstention and deferral \citep{mohri2023learning, mao2023twostage, yang2023alignment, cheng2024ai}.

\begin{figure}[t!]
    \centering
    \includegraphics[scale=0.22]{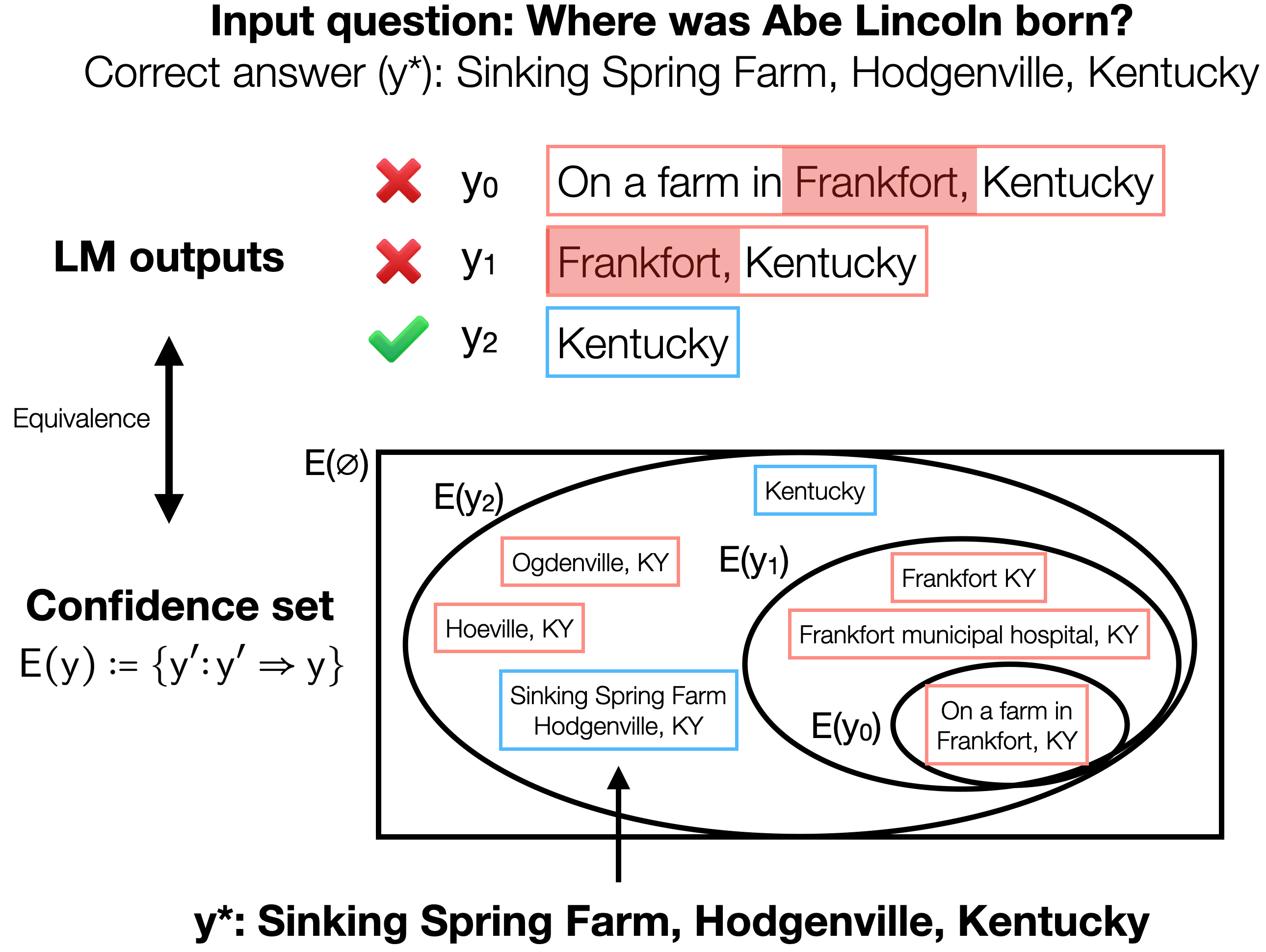}
    \caption{Conformal factuality uses conformal prediction to ensure the correctness of LM outputs. Each potential LM output sequence (top) is associated with an uncertainty set (bottom) that contains every `more specific' statement that entails it. Conformal prediction provides probabilistic guarantees that these uncertainty sets contain a correct answer (blue), which in turn guarantees the correctness of the associated output.}
    \label{fig:conf-equiv}
\end{figure}

While the factuality and correctness of language models are improving, precise guarantees are still needed. In many domains such as health \citep{tang-etal-2023-aligning, Thirunavukarasu2023, li2023chatdoctor}, law \citep{huang2023lawyer, curran2023hallucination}, or robotics \citep{zeng2023large, yang2023plug}, safely deploying a language model requires outputs to be correct with at least some known, user-specified probability. However, the complexity and opacity of LLMs make it challenging to provide precise performance guarantees.

To enable such high-probability correctness guarantees for black-box LLMs, we take inspiration from conformal prediction \citep{shafer2008tutorial, angelopoulos2022gentle}, a framework that performs uncertainty quantification 
on black-box machine learning systems without strong distributional or modeling assumptions. For any input, this framework produces conformal sets that are guaranteed to have marginally valid \emph{coverage}, meaning that on average they contain a correct output with any user-specified probability. 

While highly successful in regression and classification tasks \citep{balasubramanian2014conformal}, conformal prediction has had limited success in LLMs for two reasons: the need to score the entire output space makes it intractable, and the resulting confidence sets are so large that they are unusable.
Although a few approaches have applied conformal prediction in LLMs for multiple-choice settings \citep{kumar2023conformal, ren2023robots} or approximations \citep{quach2023conformal}, no existing method can provide exact, conformal guarantees on the open-ended outputs of language models. In our work, we propose a new natural correspondence between conformal sets and an LLM's outputs -- this not only resolves the major challenges above, but it will also directly lead to useful LLM-based systems that have correctness guarantees.

The key insight of our work is that each possible LM output defines an associated uncertainty set, where this set is defined as the set of statements that entail the LM's output. 
Under this definition, there is a direct correspondence between \emph{correctness} and \emph{coverage}, since containing a correct response in the uncertainty set implies that the associated LM's output must also be correct by entailment~(\Cref{fig:conf-equiv}). 
Defining the conformal sets \emph{implicitly} via entailment relations makes conformal prediction for language models practical and useful, as we never instantiate uncountably large uncertainty sets and we obtain meaningful and interpretable guarantees directly on model outputs.

Using the correspondence between LM outputs and their associated uncertainty sets, we show how conformal prediction defines a back-off algorithm for ensuring the correctness of LM outputs. This algorithm provides a high-probability correctness guarantee on outputs by producing a chain of output sequences that are increasingly less specific claims and then selects a level of specificity that is correct with high probability, using standard techniques from conformal prediction \citep{Gupta_2022}. 

While our approach provides guarantees for \emph{any} black box LM output, we demonstrate the practical utility of our approach by 
providing correctness guarantees on GPT-4~\citep{openai2023gpt4} outputs. Our method works by taking the outputs of GPT-4 and repeatedly removing the least certain sub-claims from the output using a GPT-4 prompt. 
Across closed-book QA and reasoning tasks, we show that conformal factuality enables us to attain any target correctness, and results in usable systems that simultaneously have correctness guarantees far higher than the correctness of the base model (30\% $\rightarrow$ 80\% on FactScore, 78\% $\rightarrow$ 93\% on NaturalQuestions, and 75\% $\rightarrow$ 95\% on MATH) while also retaining the majority of the sub-claims in the output.

\ignore{\citep{bates2021distributionfree}, we do not require a representation of this hierarchy through an explicit tree, and our set-valued predictors do not rely on computing scores of leaf nodes. }

We summarize our main contributions below.
\begin{itemize}
    \item We develop a natural correspondence between conformal prediction and LMs using entailment to define uncertainty sets associated with LM outputs.
    \item We provide an algorithmic instantiation of conformal factuality by breaking down LM outputs into sub-claims, scoring them, and removing claims according to their uncertainty.
    \item We demonstrate that our conformal factuality instantiated with GPT-4 can provide high-probability correctness guarantee on closed book QA and reasoning tasks while retaining most of the sub-claims in the outputs.\footnote{We release our code at \url{https://github.com/tatsu-lab/conformal-factual-lm}.}
\end{itemize}

The rest of this paper is organized as follows. We discuss related works in detail in Section~\ref{sec:related} followed by the desired guarantee and present our algorithm in Section~\ref{sec:prelim}.  In Section~\ref{sec:theoretical_analysis}, we present a theoretical analysis of our algorithm. Section~\ref{sec:implementation} describes our implementation, and Section~\ref{sec:experiments} shows the efficacy of our approach on both QA and reasoning tasks. 

\section{Related Work}\label{sec:related}
Our work relates to a growing literature on factuality for LLMs, conformal prediction, and conformal prediction for language generation.

\paragraph*{LM factuality. } The factuality of language models is a major concern and a topic of significant research interest. We refer our reader to the several surveys covering research on reducing, detecting, and evaluating hallucinations for a broader view~\citep{Ji_2023,huang2023survey,wang2023survey} and mention a few relevant works below. 

One class of works aims to improve LLMs' access to knowledge, via mechanisms such as retrieval augmentation~\citep{lewis2021retrievalaugmented,Das19,karpukhin2020dense} while another seeks to suppress LLMs' tendency to fabricate facts by adjusting the training method~\citep{Kang20}, ensembling~\citep{Kuhn2023SemanticUL}, or modifying the decoder~\citep{Shi2023TrustingYE,yang2023alignment}. Our work falls in the latter group, and shares some similarities to methods such as semantic uncertainty in that it is a decoding-time method, but has different goals in that we aim for precise, probabilistic guarantees of correctness. 

Our work also relates to a line of work on uncertainty quantification for LLMs, as we must be able to identify and remove unreliable sub-claims from an LLM's output. 
The implementation of conformal factuality in this work uses a method closely related to SelfCheckGPT~\citep{manakul2023selfcheckgpt} as well as prompting the LM for uncertainty estimates~\citep{tian2023just}. These works are complementary to ours, as our contribution is a meta-algorithm that uses these uncertainty estimators to return LM outputs that have conformal factuality guarantees. We expect future developments in this line of work to improve the performance of our methods as our algorithms work with any uncertainty quantification method.

\paragraph*{Conformal prediction. }
Conformal prediction is a statistical technique for constructing confidence sets with precise, marginal coverage guarantees without strong distributional or modeling assumptions \citep{Shafer08,angelopoulos2022gentle,balasubramanian2014conformal,barber2023conformal}. This approach has been successful in providing confidence sets for black-box models such as deep neural networks \citep{einbinder2022training,balasubramanian2014conformal} but its application to language models has been limited.

In language models, there are three major families of applications of conformal prediction to language models: token-based approaches construct confidence sets on \emph{individual tokens}~\citep{ravfogel2023conformal,ulmer2024nonexchangeable}. This constrains the prediction space and makes it possible to apply standard conformal prediction techniques, but coverage guarantees over tokens cannot be converted into correctness guarantees for sequences. Multiple-choice reduction approaches reduce the prediction space of the LM in constrained domains like question answering. These approaches provide coverage guarantees over the output (e.g. the confidence sets contain the answer with probability $1-\alpha$) but can only be applied in highly constrained QA-style domains \citep{kumar2023conformal}. Finally, recent work has attempted to instantiate conformal uncertainty sets directly on the space of sequences~\citep{quach2023conformal} but the space of all sequences is intractably large, and this necessitates approximations to the true confidence set. Most importantly, all of these approaches return \emph{sets} of tokens and sequences, which are difficult to interpret and act upon. In contrast, our conformal factuality approach returns a \emph{single natural language sequence} (which implicitly represents a confidence set) and thus can be used directly to improve the output of LMs.

Finally, closest to our work is \citet{angelopoulos2023conformal}, which gives a conformal prediction algorithm that can be applied to classification problems with hierarchical labels. In that setting, their approach can return an intermediate node in the tree which implicitly represents a confidence set consisting of all the leaves of this subtree. While our work is similar in that we implicitly represent confidence sets for conformal prediction, our work differs in the construction of the set (via entailments), representation of the hierarchy (we do not enumerate
its edges), the setting and implementation (language models) and the actual conformal guarantee (correctness of an output rather than risk control).

\section{Preliminaries}\label{sec:prelim}
In the standard language model (LM) generation setting, we receive an input $x\in \sX$ and generate an output $y \in \sY$ according to a (potentially stochastic) generator $y=\sfL(x)$. A key problem is that $y$ may not be fully supported by a ground truth or reference $y^* \in \sY$ \citep{maynez2020faithfulness, huang2023survey}. 

\noindent\textbf{Our goal in this work} is to provide precise control over the correctness and factuality of an LM's output. As it may be difficult (if not impossible) to provide guarantees on every LM output, our goal will be to provide high-probability guarantees such that for any user-specified probability $\alpha \in (0,1)$, the LM is correct with probability at least $(1-\alpha)$ over some distribution $\P$. We express this goal as
\begin{align}
\label{eq:goal}
    \P( y  \text{ is factual and correct} ) \geq 1-\alpha.
\end{align}

Throughout this work, we will formalize this correctness constraint in terms of entailments \citep{maccartney2014natural} with respect to some reference knowledge $y^*$ where correctness is equivalent to the entailment relation $y^* \Rightarrow y$.\footnote{Entailment can be ambiguous, and our work provides guarantees for any definition of entailment (such as entailment as judged by domain experts, crowd workers, or even an automated fact checker), as long as the user has access to a binary entailment oracle. Our guarantees only require that $\forall y\in\sY, y \Rightarrow \emptyset$. } Representing factuality and correctness via entailments to a reference is quite general, as we can set $y^*$ to be a broad knowledge base such as `Wikipedia pages related to $x$' or even `all facts accessible via Google' to handle the case where there is no ground truth response for $y^*$. A key equivalence in our work will be that the factuality constraint $y^* \Rightarrow y$ can be written as a set containment relation $y^* \in \{y' \in \sY \colon y' \Rightarrow y\}$. We simplify this set using the \emph{entailment operator} $\sfoE\colon \sY \mapsto 2^{\sY}$ defined by 
\[\sfoE(y) := \{y'\in\sY \colon y' \Rightarrow y \},\] such that $y^*\in \sfoE(y)$ holds if and only if $y^*\Rightarrow y$.

\ignore{\begin{align*}
    \sfoE(\mathsf{y}) := \{\mathsf{y}' \colon \mathsf{y}' \Rightarrow \mathsf{y} \}
\end{align*}
}

Using this entailment set, we can now begin to connect our goal (\ref{eq:goal}) to a well-studied statistical inference problem: for exchangeable $(X_i, Y_i^*) \in \sX \times \sY, i \in [n+1]$, use $\{(X_i, Y_i^*) \}_{i=1}^n$ to find some uncertainty set $C \colon \sX \mapsto 2^\sY$ such that the ground truth $Y_{n+1}^*$ satisfies the following inequality:
\begin{align}\label{conformal_guarantee}
    \P(Y_{n+1}^*\in C(X_{n+1})) \geq 1-\alpha.
\end{align}
The connection between this inference problem and our correctness goal becomes clear if we replace the uncertainty set $C(X_{n+1})$ with the entailment set $\sfoE(\sfL(X_{n+1}))$ of the LM output $\sfL(X_{n+1})\in \sY$. In this case, the inequality~(\ref{conformal_guarantee}) would give precisely our correctness bound~(\ref{eq:goal}) since the event `Containing $Y^* \in \sfoE(\sfL(X))$' is equivalent to the event `$\sfL(X)$ is correct according to $Y^*$'. In the subsequent sections, we give formal constructions of our uncertainty sets and introduce conformal prediction techniques for performing the inference problem.


\noindent\textbf{Background on split conformal prediction.} 
Split conformal prediction~\citep{shafer2008tutorial, Gupta_2022} provides standard tools by which we can construct $C(\cdot)$ that satisfy the constraint in inequality~(\ref{conformal_guarantee}). Our work follows the standard split conformal prediction approach in \citep{Gupta_2022}, where one constructs a sequence of nested sets and uses exchangeable calibration data to pick a nested set that is sufficiently large to fulfill the inequality in~(\ref{conformal_guarantee}). Formally, for a threshold set $\cT \subseteq \Rset$ and each input $x\in\sX$, let $\{\cF_t(x)\}_{t\in\cT}$ denote a sequence of output
sets following the \emph{nested property}, meaning that $\cF_t(x) \subseteq \cF_{t'}(x)$ for $t \leq t'$. Consider the score
\begin{align}\label{def:rscore}
    r(x, y) := \inf \{t\in \cT \colon y \in \cF_t(x) \}.
\end{align}
This can be thought of as the minimum \emph{safe} threshold where $y \in \cF_t(x)$ for every $t>r(x,y)$.
Split conformal prediction then sets the final confidence set as
\begin{align*}
    C(x) = \cF_{\hat q_\alpha}(x),    
\end{align*}
where $\hat q_\alpha$ is the $\frac{\ceil{(n+1)(1-\alpha)}}{n}$th quantile of the scores $\{r(X_i, Y_i)\}_{i=1}^n$. This implementation satisfies the constraint in inequality~(\ref{conformal_guarantee}) (see Proposition 1 of \citep{Gupta_2022} for a proof). We now show how to generalize $\cF_t$ for language models in a way that also leads to factuality and correctness guarantee on LM outputs. 

\noindent\textbf{Application to the language setting.} 
Recall that the correctness of an LM output $y$  is equivalent to the event $y^* \in \sfoE(y)$, and we seek to find some $y$ that makes this event hold with probability at least $1-\alpha$. 
To do this, we construct sequences of outputs $\{y_t\}_{t\in\cT}$ which induce sequences of associated conformal sets $\{\sfoE(y_t)\}_{t\in\cT}$ on which we can apply conformal prediction methods. While these sets could be nested like in \citet{Gupta_2022} (implying that the associated outputs $y_t$ become strictly `more generic' as in Figure~\ref{fig:conf-equiv}), this constraint can be hard to enforce for a language model, and we show that our main guarantees do not require nestedness. 

In conformal prediction, the key objects
are the sequence of conformal sets $\{\cF_t\}_{t\in\cT}$ and the score $r(x,y)$. We will define these two quantities for the LM setting below. 

For the conformal sets, we will define these sets using the entailment operator $\sfoE$ as $\cF_t(x) = \sfoE(\SL_{t}(x,\sfL(x)))$
where $\SL_t: \sX \times \sY \mapsto \sY$ is a `back off' function and the threshold $t\in\cT\subseteq\Rset$ controls how much $\SL_{t} (x, y_0)$ `backs off' from the base output $y_0$ by removing (unreliable) claims. We call $\SL_t$ \emph{sound} if it satisfies the property that $\SL_{\sup \cT} (x, y_0) = \emptyset$, where $\emptyset$ represents some output sequence that abstains from making any claim. For notational clarity, we will omit the second argument whenever there is only one relevant language model $\sfL(x)$ that can generate $y_0$. In this case, we use the shorthand $\SL_t(x):=\SL_t(x,\sfL(x))$.

\ignore{
\begin{figure}
    \centering 
    \begin{align*}
    \SL_0(x) = & \texttt{ Michael Jordan \hl{(born February 17,}} \\& \texttt{ \hl{1063)}, is a former professional } \\& \texttt{ basketball player.}\\
    \SL_1(x)  = & \texttt{ Michael Jordan is a former } \\& \texttt{ professional basketball player. } \\
    \SL_2(x) = & \texttt{ } \emptyset.
    \end{align*}
    \caption{Sound implementation of $\SL_t$ with $\cT = \{0, 1, 2\}$ and input sequence $x$ = $\texttt{Tell me a biography of Michael Jordan}$. Note that an incorrect date of birth is presented in $\SL_0(x)$, but not in $\SL_1(x)$ and $\SL_2(x)$; $\SL_t$ becomes factual for large enough $t \in \cT$.}
    \label{fig:sl_example}
\end{figure}
}

For the score function, we can redefine the score in~(\ref{def:rscore}) as
\begin{align}\label{new_rscore}
    r(x, y^*) := \inf \{ t\in\cT \colon \forall j \geq t, y^* \in \sfoE(\SL_j(x)) \}.
\end{align}
This matches the original score with one minor modification where we take the minimum \emph{strictly safe} threshold---we consider a threshold \emph{strictly safe} if any threshold greater than or equal to this one is safe.\footnote{The key difference with respect to \citep{Gupta_2022} is that in the definition of $r$ we write: $\forall j \geq t$. This is implicitly encoded in their definition due to the \emph{nested property} of their set predictors, but as we do not require nestedness, we instead explicitly modify our thresholds to be strictly safe.} For the example in Figure~\ref{fig:conf-equiv}, if we add $y_3 = \emptyset$ and define $\SL_t(x) := y_t$, we would have the minimum strictly safe threshold $r(x, y^*) = 2.$ 

With these two components in hand, we can directly apply the split conformal prediction method to obtain an LM with our desired correctness guarantees in inequality~\ref{eq:goal}. Formally, we say that a model $\bar \sfL$ is \emph{$\alpha$-conformally factual} if for exchangeable $(X_i, Y_i^*) \in \sX \times \sY, i \in [n+1]$ and $\{(X_i, Y_i^*) \}_{i=1}^n$ used to construct $\bar \sfL$, the reference output $Y_{n+1}^*$ satisfies the following inequality:
\begin{align*}
    \P(Y_{n+1}^*\in \sfoE(\bar \sfL(X_{n+1}))) \geq 1-\alpha.
\end{align*}
\ignore{In summary, split conformal prediction builds nested sets $\cF_t(X_{n+1})$ indexed by $t$ and a calibration set $\{(X_i, Y_i^*) \}_{i=1}^n$ selects $t$ to contain $Y^*_{n+1}$ with high probability. In our construction, we build a sequence of statements $\{\SL_t(X_{n+1}) \}_{t\in\cT}$, and a similar calibration set selects $t$ such that $\sfoE(\SL_t(X_{n+1}))$ contains $Y^*_{n+1}$ with high probability.}

\begin{algorithm}[t!]
\caption{$\alpha$-conformal-factuality algorithm}\label{algo:conformally_factual_algo}
\textbf{Inputs: }{base LM $\sfL\colon \sX \mapsto \sY$, confidence $\alpha$, calibration data $\{X_i , Y_i^*\}_{i=1}^{n}$, and back-off mechanism $\{\SL_t\}_{t\in\cT}$} \\ \vspace{0.1cm}
\For{$i \leftarrow 1$ \KwTo $n$}{
    $r(X_i, Y_i^*) \gets  \inf \{ t\in\cT \colon \forall j \geq t, y^* \in \sfoE(\SL_j(x, \sfL(x))) \}$\
}
$\hat q_\alpha \gets \frac{\ceil{(n+1)(1-\alpha)}}{n}$th quantile of the scores $\{r(X_i, Y_i^*)\}_{i=1}^n$\\  \vspace{0.1cm}
\textbf{Output:} conformally factual $\bar \sfL(x) := \SL_{\hat q_\alpha}(x,\sfL(x))$.
\end{algorithm}

We present our algorithm for achieving $\alpha$-conformal-factuality in Algorithm~\ref{algo:conformally_factual_algo}, which is a procedure that takes in a base LM $\sfL\colon \sX \mapsto \sY$, target error rate $\alpha$, back-off mechanism $\SL_t$ and a calibration dataset, and produces a new LM $\bar \sfL$ that is $\alpha$-conformally factual. In the following sections, we prove high-probability factuality guarantees for Algorithm~\ref{algo:conformally_factual_algo} (Section~\ref{sec:theoretical_analysis}), and provide our implementation of $\SL_t$ (Section~\ref{sec:implementation}) with experiments on several datasets (Section~\ref{sec:experiments}).

\section{Theoretical analysis}
\label{sec:theoretical_analysis}
In this section, we present a theoretical analysis of Algorithm~\ref{algo:conformally_factual_algo}, giving upper and lower bounds matching those of standard split conformal prediction and providing a guarantee of the form in inequality~(\ref{eq:goal}).

\begin{restatable}{theorem}{CpEquiv}
\label{th:CpEquiv}
Let $\{X_i, Y_i^*\}_{i=1}^{n+1}$ be exchangeable, $\SL_t$ be sound, and $\hat q_\alpha$ be defined as the $\frac{\ceil{(n+1)(1-\alpha)}}{n}$th quantile of the scores $\{r(X_i, Y^*_i)\}_{i=1}^n$, which we assume to be distinct without loss of generality. Then, for $\alpha \in \left[\frac{1}{n+1}, 1\right]$, the following lower bound holds:
\begin{align*}
    \P( Y_{n+1}^* \in \sfoE(\SL_{\hat q_\alpha}(X_{n+1}))  ) \geq 1-\alpha.
\end{align*} 
If $\sfoE(\SL_t(\cdot))$ follows the nested property, then the following upper bound holds:
\begin{align*}
    1-\alpha + \frac{1}{n+1} \geq \P( Y_{n+1}^* \in \sfoE(\SL_{\hat q_\alpha}(X_{n+1}))  ) .
\end{align*} 

\end{restatable}
\begin{proof}
    Let $r_i = r(X_i, Y_i^*)$ for $i \in [n]$ and $r_{\text{test}} = r(X_{n+1}, Y_{n+1}^*)$. These scores are all well-defined because $\SL_{\sup \cT} = \emptyset$ and $\forall y\in\sY, y \Rightarrow \emptyset$. Without loss of generality, we can assume that the scores are sorted $r_1 < r_2 < \ldots < r_n$. In that case, $\hat q_\alpha = r_{\ceil{(1-\alpha)(n+1)}}$ when $\alpha \geq \frac{1}{n+1}$. We note that by exchangeability,
    \begin{align*}
        \P(r_{\text{test}} \leq r_{\ceil{(1-\alpha)(n+1)}}) = \frac{\ceil{(1-\alpha)(n+1)}}{n+1} \geq 1-\alpha.
    \end{align*}
    We now observe the relationship between the following two events: 
        \begin{align*}
        \{r_{\text{test}} \leq \hat q_\alpha\} \text{ implies } \{ Y_{n+1}^* \Rightarrow\SL_{\hat q_\alpha }(X_{n+1})  \},
    \end{align*}
    because if $r_{\text{test}} \leq \hat q_\alpha$, then $\hat q_\alpha$ is a safe threshold. This completes the proof of the lower bound. 
    Now, since for any $a\in\Rset, \ceil{a}\leq a+1$, we obtain the upper bound:
    \begin{align*}
        \P(r_{\text{test}} \leq r_{\ceil{(1-\alpha)(n+1)}}) & = \frac{\ceil{(1-\alpha)(n+1)}}{n+1} \\ & \leq \frac{(1-\alpha)(n+1) + 1}{n+1} \\ &  = 1-\alpha + \frac{1}{n+1}.
    \end{align*}
    If $\sfoE(\SL_t(\cdot))$ follows the \emph{nested property}, we now observe the equality of two events: 
        \begin{align*}
        \{r_{\text{test}} \leq \hat q_\alpha\} = \{ Y_{n+1}^* \Rightarrow\SL_{\hat q_\alpha }(X_{n+1})  \},
    \end{align*}
    as $\hat q_\alpha$ being a safe threshold now implies that it is larger than or equal to the minimum strictly safe threshold. This completes the proof of the upper bound. 
\end{proof}
Thus, Algorithm~\ref{algo:conformally_factual_algo} achieves $\alpha$-conformal factuality for any user-specified correctness target with $\alpha \in \left[\frac{1}{n+1}, 1\right]$, along with an upper bound when the nested property holds. Remarkably, we can guarantee that output sequences in $\SL_{\hat q_\alpha}(\cdot)$ are factual with high probability over exchangeable sequences. While we can always obtain this guarantee, it does not necessarily imply that we can retain the usefulness of the LM outputs---the threshold $\hat q_\alpha$ may be so large that they are uninformative or even empty. In the next section, we provide an implementation of $\SL_t$ that aims to keep this threshold small. 

\ignore{
\textcolor{red}{Randomness in output?} If the output of $\SL$ is meant to be random (either from the LM it implicitly represents or the way in which it removes content), it can take an additional argument $\omega \in \Omega$ to encode the randomness, which would be drawn independently along with any $X_i \sim \sD$. 
}

\section{Implementation of $\SL_t$ via sub-claims}
\label{sec:implementation}

\begin{figure}[t]
  \centering
  \includegraphics[width=0.6\textwidth]{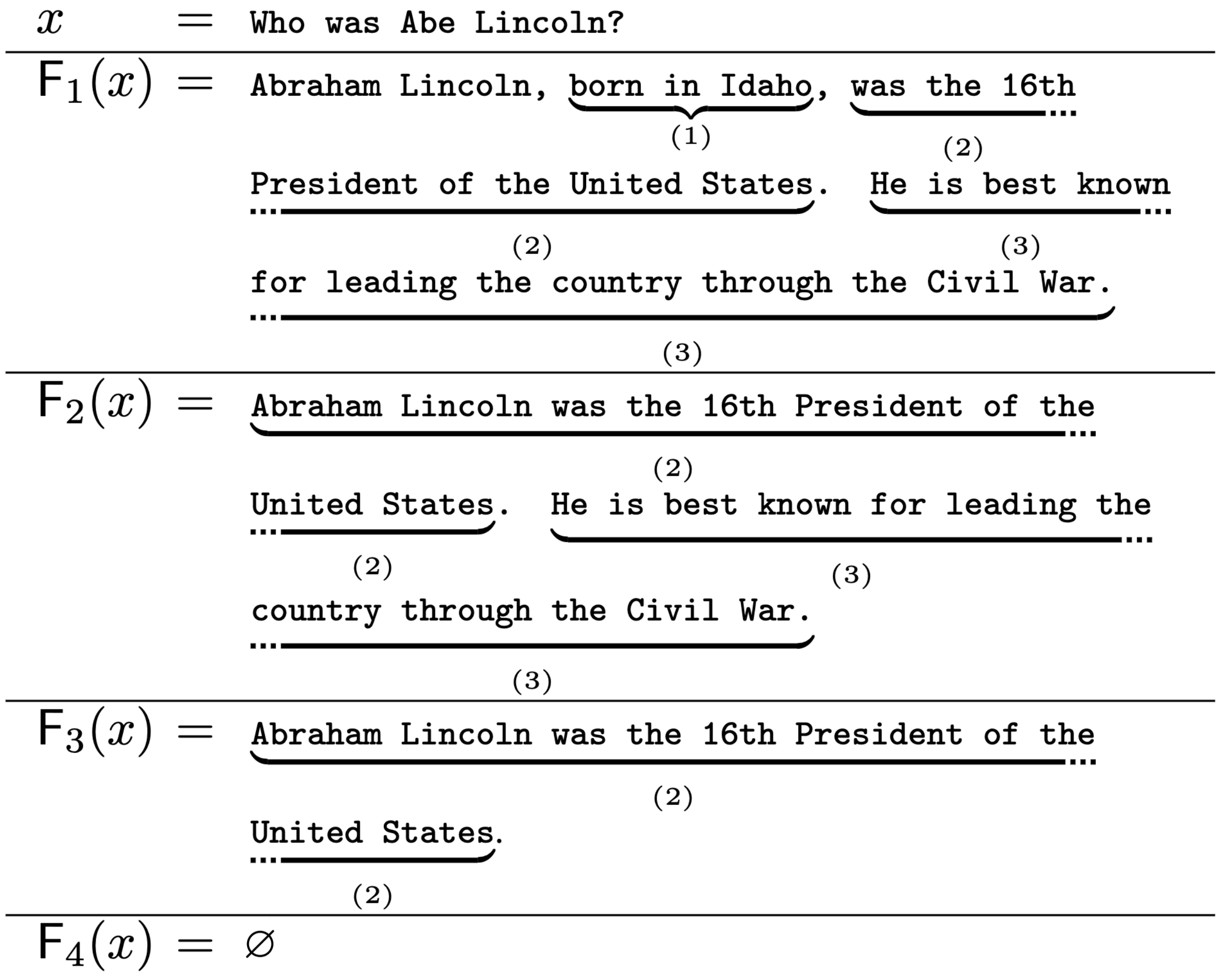}
  \ignore{\begin{tabular}{p{0.55cm}p{0.1cm}p{6.5cm}}
    $x$ & $=$ & \scriptsize{\texttt{Who was Abe Lincoln?}} \\ \hline
    $\SL_1(x)$ & $=$ &   \scriptsize{\texttt{Abraham Lincoln, $\underbrace{\texttt{born in Idaho}}_{(1)}$, $\underbracea{\texttt{was the 16th}}_{(2)}$ $\underbraced{\texttt{President of the United States}}_{(2)}$. $\underbracea{\texttt{He is best known}}_{(3)}$ $\underbraced{\texttt{for leading the country through the Civil War.}}_{(3)}$ }}
  \\ \hline
    $\SL_2(x)$ & $=$ & \scriptsize{\texttt{$\underbracea{\texttt{Abraham Lincoln was the 16th President of the}}_{(2)}$ $\underbraced{\texttt{United States}}_{(2)}$.  $\underbracea{\texttt{He is best known for leading the}}_{(3)}$ $\underbraced{\texttt{country through the Civil War.}}_{(3)}$}} \\ \hline
    $\SL_3(x)$ & $=$ & \scriptsize{$\underbracea{\texttt{Abraham Lincoln was the 16th President of the}}_{(2)}$ $\underbraced{\texttt{United States}}_{(2)}$}. \\ \hline
    $\SL_4(x)$ & $=$ & $\emptyset$ \\
  \end{tabular}}
  \caption{Example $\{\SL_t(x)\}_{t\in\cT}$ via sub-claims. Here we identified three sub-claims corresponding to (1) Abe Lincoln's birthplace, (2) his notable job, and (3) what he was best known for.  }
  \label{fig:f_t_via_subclaims_example}
\end{figure}

Our guarantees hold with any sound $\SL_t$, but ideally, $\SL_t$ should first remove unreliable parts of an output sequence as the threshold $t$ increases. We now construct an empirically effective instantiation of $\SL_t$ that makes use of the following observation: the LM often confidently knows that some subparts of its answer are correct, so it often suffices to remove the `uncertain' subparts to balance correctness and usefulness. We start by defining our implementation of $\SL_t$ inspired by this idea, and then we analyze the simple procedure that it admits for computing scores $r$. 

Our implementation identifies unreliable parts of an output sequence by decomposing it into sub-claims. Let $\sfL \colon \sX \to \sY$ be a mapping derived from a language model, $\sfS\colon \sY \mapsto 2^{\sY}$ be a function that separates an output sequence into sub-claims, and $\sfM \colon 2^{\sY} \mapsto \sY$ be a function that merges sub-claims into a single sequence and satisfies $\sfM(\emptyset) = \emptyset$. Let $s\colon 2^{\sY} \times \sY \mapsto \Rset$ be a sub-claim scoring function, where a larger score is meant to denote a larger probability of a sub-claim being factual. \ignore{\footnote{We require the first argument to allow for implementations that compare a sub-claim to other sub-claims in an output sequence.}} Intuitively, we merge a set of extracted sub-claims that were scored at least $t$. We implement $\SL_t$ as follows:
\begin{align} \label{def:sl_subclaims}
    \SL_t(x) = \sfM(\{ c \in (\sfS \circ \sfL)(x) \colon s((\sfS \circ \sfL)(x), c) \geq t \}).
\end{align} 
This implementation is \emph{sound}, as no sub-claims are accepted for large enough $t$. We provide an example in Figure~\ref{fig:f_t_via_subclaims_example} to show how a sequence of $\SL_t$ might look. Note that under the assumptions of Theorem~\ref{th:CpEquiv}, applying Algorithm~\ref{algo:conformally_factual_algo} with $\SL_t$ implemeted as in~(\ref{def:sl_subclaims}) leads to $\alpha$-conformal factuality. We will also see that the upper bound of Theorem~\ref{th:CpEquiv} holds under a simple assumption on $\sfM$. 

\renewcommand{\arraystretch}{1.5} 

One additional advantage of implementing $\SL_t$ via subclaims is that it can substantially reduce annotation effort. Normally, to compute the infimum in the definition of $r$, one has to evaluate entailment across all outputs $\{\SL_t(\cdot)\}_{t\in\cT}$. However, we now show that this can be done much more cheaply by only evaluating the entailment of the sub-claims once and computing an infimum over the sub-claims. This trick of computing entailments on sub-claims preserves all our guarantees under the natural assumption that the merger function $M$ that does not add or remove any sub-claims, thus preserving the entailment relations between the sub-claims and the merged output.
\begin{restatable}{assumption}{MAsssumption}
\label{assumption_m}
For any $y^*\in \sY$ and $\{c_i\}_{i=1}^n \in 2^{\sY}$,
\begin{align*}
    \{y^* \Rightarrow \sfM(\{c_i\}_{i=1}^n) 
    \} \iff  \{ \forall i \in [n], y^* \Rightarrow c_i \},
\end{align*}
which could equivalently be written as:
\begin{align*}
    \sfoE(\sfM(\{c_i\}_{i=1}^n)) = \bigcap_{i=1}^n \sfoE(c_i).
\end{align*}    
\end{restatable}
In this case, the $r$ admits a simpler form. 
\begin{restatable}{proposition}{ProSco}
\label{proposition:score}
    For $x\in \sX$, let the sub-claims accepted by $\SL_t$ be denoted by $A_t(x) := \{ c \in (\sfS \circ \sfL)(x) \colon s((\sfS \circ \sfL)(x), c) \geq t \} \in 2^{\sY}$. Under Assumption~\ref{assumption_m}, $r(x, y^*)$ can be computed as 
\begin{align*}
    r(x, y^*) = \inf \{ t\in\cT \colon \forall j \geq t, \forall c \in A_t(x),  y^* \Rightarrow c \}.
\end{align*} 
\end{restatable}
\begin{proof}
We first observe the following equivalence:
\begin{align*}
      \{y^* \in \sfoE(\SL_t(x)) \} & \iff \{y^* \Rightarrow \SL_t(x) \}  \\
     & \iff \{ y^* \Rightarrow \sfM(A_t(x)) \} \\
     & \iff \{\forall c \in A_t(x),  y^* \Rightarrow c \}. 
\end{align*}
Using the definition of $r$, we can write
\begin{align*}
    r(x, y^*) :&= \inf \{ t\in\cT \colon \forall j \geq t, y^* \in \sfoE(\SL_j(x)) \} \\ & = \inf \{ t\in\cT \colon \forall j \geq t, \forall c \in A_t(x),  y^* \Rightarrow c \},
\end{align*}
which completes the proof. 
\end{proof}
Thus, if Assumption~\ref{assumption_m} holds, one only has to call the entailment oracle on the sub-claims appearing in the original output, which can be significantly cheaper than calling the entailment oracle on the merger of every possible set of accepted sub-claims.

Finally, we note that Assumption~\ref{assumption_m} also gives us the upper bound of Theorem~\ref{th:CpEquiv} since $\SL_t$ follows the nested property---we have $\sfoE(\SL_t(x)) = \sfoE(\sfM(A_t(x))) = \bigcap_{c\in A_t(x)} \sfoE(c)$ and an intersection of sets becomes larger as sets are removed. 

\ignore{
\textcolor{red}{If all the facts are accepted, and we make a similar assumption on $S$, we don't have to remerge, we can just take the original output.}
}

\subsection{Partial entailment}\label{subsec:partial_entailment}
Finally, we note that our framework can be extended to provide guarantees for partial correctness. Instead of guaranteeing full factuality, one may want to guarantee that $a\in[0, 1]$ fraction of the accepted sub-claims are factual. To achieve this, we can modify the definition of $r$ to allow for partially-entailed sets of sub-claims. Let $\sT_{y^*}: 2^\sY \mapsto [0, 1]$ denote an operator indicating the entailed fraction of a set of sub-claims. That is, $\sT_{y^*}( \{s_{i}\}_{i=1}^m) := \frac{1}{m}\sum_{i=1}^m \1_{y^* \Rightarrow s_i}$. Define a new score with acceptable entailment level $a \in [0, 1]$: 
\begin{align*}
    r_a(x, y^*)  := \inf \{ t\in\cT \colon \forall j \geq t, \sT_{y^*}(A_t(x)) \geq a  \}.
\end{align*}
Note that $r_1(\cdot) = r(\cdot)$ as defined in~(\ref{new_rscore}). Then, we obtain a result similar to the lower bound of Theorem~\ref{th:CpEquiv}. 
\begin{restatable}{corollary}{CorPar}
\label{corollary:par} (Partial entailment via sub-claims)
    Let the assumptions of Theorem~\ref{th:CpEquiv} hold, but with $\hat q_\alpha$ as the $\frac{\ceil{(n+1)(1-a)}}{n}$ quantile of the scores $r_a(X_i, Y_i^*)$ for $a\in[0, 1]$ and $\SL_t$ implemented as in~(\ref{def:sl_subclaims}). Then, the following lower bound holds:
\begin{align*}
    \P(\sT_{Y_{n+1}^*}( A_{\hat q_\alpha}(X_{n+1})) \geq a) \geq 1-\alpha.
\end{align*} 
\end{restatable}
\begin{proof}
    The proof is identical to lower bound of Theorem~\ref{th:CpEquiv}, but now we note the following relationship between two events (where $r_{\text{a, test}} := r_a(X_{n+1}, Y_{n+1}^*))$:
    \begin{align*}
        \{r_{\text{a, test}} \leq \hat q_\alpha\} \text{ implies } \{\sT_{Y_{n+1}^*}( A_{\hat q_\alpha}(X_{n+1})) \geq a \}.
    \end{align*}
    This is not an equivalence, because we could have a safe threshold $\hat q_\alpha$ that is less than $r_{\text{a, test}}$. That is, $\sT_{y^*}(A_t(x))$ is not monotonically increasing with $t$ since one could remove sub-claims that are entailed and fall below $a$.
\end{proof}
Note that this result is with respect to the sub-claims themselves, rather than the merged sub-claims, but simple approaches such as providing the sub-claims alongside the output can enable users to verify the correctness of $M$.\footnote{We leave this result in terms of the sub-claims since there is no agreed-upon sub-claim separator function and it would not make sense to say that an output sequence is $a$-fraction factual without one.}\ignore{, and if one were to use $S$, computing the score $r$ may require re-separating the merged sub-claims unless strong assumptions such as $S$ and $M$ being inverses of each other are made.} 
Empirically, we found gains from allowing partial factuality to be small, and cover empirical evaluations of this class of approaches in Appendix~\ref{app:partial-factuality}. 

\ignore{
\subsection{$\SL$ algorithm via Sub-claims.}

We describe in detail the $\SL$ algorithm under this implementation of $\SL$.

\begin{enumerate}
    \item Identify definitions of $y^*$, $\Rightarrow$, $L$, $S$, $M$, and a scoring function of sub-claims $s \colon 2^\sY \times \sY \mapsto \Rset$, where a larger score should denote a larger probability that a sub-claim is entailed. $L$, $S$, and $M$ can for example be derived by prompting a language model. 
    \item Over a held-out sample $\{X_i\}_{i=1}^n$, compute the $\frac{\ceil{(n+1)(1-\alpha)}}{n}$th quantile of the scores $r_a(X_i, Y_i^*)$. 
    \item Use that quantile to form new output sequences via $\SL$.
\end{enumerate}
}

\section{Experiments}
\label{sec:experiments}

While we provide conformal factuality guarantees on LM outputs derived from Algorithm~\ref{algo:conformally_factual_algo}, we still need to verify that they are indeed factual and useful. To verify these two items, we apply our algorithm to standard question-answering and reasoning tasks where correctness and factuality guarantees would be useful. We first explicitly describe our experimental set-up, and then present results across $3$ datasets. 

\ignore{
Note: there is a bug in \citep{angelopoulos2022gentle} example code, in the computation of the quantile. This is a problem when there is a small amount of data.}

\subsection{Experimental set-up}

Here we instantiate all the pieces necessary to implement $\SL_t$ via sub-claims and describe our datasets. 

\subsubsection{Models}

The definition of $\SL_t$ via sub-claims in~(\ref{def:sl_subclaims}) depends on multiple language models: a base mapping from input sequences to output sequences $\sfL$, a sub-claim separator $\sfS$, and a merger function $\sfM$. As a proof of concept, we implement each of these by using GPT-4. $\sfL$ is implemented by using GPT-4 directly, and $\sfS$ and $\sfM$ by using GPT-4 with prompts to separate and merge (which we present in Appendix~\ref{app:prompts}).

\subsubsection{Sub-claim scoring functions}

Recall that any sub-claim scoring function $s(\cdot)$ introduced in Section~\ref{sec:implementation} leads the guarantees of Theorem~\ref{th:CpEquiv}. Below, we provide the definitions of several natural ones that are used in our experiments.  For some input and ground truth pair $(x, y^*) \in \sX \times \sY$, let the ordered set of $n$ extracted sub-claims be $\{c_i\}_{i=1}^n = (\sfS \circ \sfL)(x)  \in 2^\sY$, sorted by where they appear in $\sfL(x)$. Below we define scoring functions for a particular subclaim $c_j \in\{c_i\}_{i=1}^n$ (and assign a score of $-\infty$ to any sub-claim not in the set $\{c_i\}_{i=1}^n$). 

We first define two baselines. 

\noindent\textbf{Random scoring.} This method assigns random scores to sub-claims and is defined as:
\begin{align*}
    s_r(\{c_i\}_{i=1}^n, c_j) = X_j, \text{ where $X_j \sim \sN(0, 1)$}.
\end{align*}

\noindent\textbf{Ordinal scoring.} This method assigns scores corresponding to the order a sub-claim appeared in an output sequence, and is defined as:
\begin{align*}
    s_b(\{c_i\}_{i=1}^n, c_j) = n - j.
\end{align*}

The following two scoring functions use an LLM like GPT-4 and are ones to consider using in practice. 

\noindent\textbf{GPT-4 confidence scoring.} Motivated by \citet{tian2023just, guan2023language}, this method directly asks GPT-4 for a confidence score. We present our prompt in Appendix~\ref{app:prompts}. \ignore{ The function \texttt{get-conf} below is implemented by prompts presented in Appendix~\ref{app:prompts}.  The score is defined as:
\begin{align*}
    s_g(\{c_i\}_{i=1}^n, c_j) = \texttt{get-conf} (\{c_i\}_{i=1}^n, c_j) .
\end{align*}}

\noindent\textbf{Frequency scoring.} Motivated by self-consistency approaches \citep{wang2023selfconsistency, manakul2023selfcheckgpt}, this method first samples $5$ alternate output sequences with temperature $1.0$ and then counts (with GPT-4) the number of times a sub-claim appeared in the alternate output sequences. The prompt used to implement this method appears in Appendix~\ref{app:prompts} as well.\footnote{In our experiments, we break the ties among these scores using the GPT-4 confidence score.} \ignore{The definitions for $\texttt{count}$ and $\texttt{sample}$ are presented in Appendix~\ref{app:prompts}. The score is defined as:
\begin{align*}
        s_f(\{c_i\}_{i=1}^n, c_j) = \texttt{count}_{c_j}(\texttt{sample}_5(x))
\end{align*}
}

Finally, to provide an upper bound on performance, we include the following oracle scoring method. 

\noindent\textbf{Oracle scoring.} This method assigns scores corresponding to true entailment. Of course, this is not possible to use without knowledge of $y^*$, and is not efficient when entailment is meant to be checked by a human. It is defined as: 
\begin{align*}
    s_o(\{c_i\}_{i=1}^n, c_j) = \1_{y^* \Rightarrow c_j}.
\end{align*}

Since our guarantees require no ties among the scores $r$, we tie-break using $\sN(0, 0.001)$ noise, and we ensure the noise terms are consistent across sub-claim scoring functions.

\ignore{
Any scoring function can be transformed into a \emph{ranking-based} scoring function. These aim to remove a fixed percentage of facts, which may be desirable in cases where modifying the outputs in roughly the same way is desirable. Here, for a sub-claim scoring function $s\colon 2^\sY \times \sY$, we define a ranking function $f_s \colon 2^\sY \times \sY \mapsto \Nset$, where $f(\{c_i\}_{i=1}^n, c_j)$ returns the rank of $s(\{c_i\}_{i=1}^n, c_j)$ among the set $\{s(\{c_i\}_{i=1}^n, c_k)\}_{k=1}^n$ in increasing order. Then, for each of sub-claim scoring functions $s$ defined above, we apply the score:  
\begin{align*}
     \frac{f_s(\{c_i\}_{i=1}^n, c_j)}{n}.
\end{align*}
We divide by $n$ to account for varying-length output sequences. \ignore{Note that the ranking-based baseline scoring function loosely matches asking the model to produce less or only considering some first fraction of its output. }}

\subsubsection{Datasets and Annotation}

We study 3 datasets covering a range of tasks that require correctness. Below, we describe both the datasets we build on and the additional factuality annotations we collect.

\noindent\textbf{FActScore} \citep{factscore}. 
FActScore is a common factuality evaluation for open-ended generation, which works by breaking a generation down into atomic facts and then evaluating them with a given knowledge source. We use the people entities from their biography generation dataset, but we generate our own sub-claims using $S$ for consistency. \ignore{Our input sequences $x$ are of 
the form $\texttt{Tell me a paragraph bio of } \text{\{person\}}$.}

\noindent\textbf{Natural Questions (NQ)} \citep{kwiatkowski-etal-2019-natural}. NQ evaluates factuality in open-ended question answering through real queries to the Google search engine. We use questions from the simplified training dataset and allow the model to respond as a long-form response.  

\noindent\textbf{MATH} \citep{hendrycks2021measuring}. This is a dataset of math word problems, and we use it to show that our correctness framework can also be applied to reasoning tasks. Answers to reasoning tasks typically involve a sequence of steps, and when we associate these steps with sub-claims, we can immediately apply our framework to return only the correct steps and abstain from the rest.

We select the first $50$ inputs from each dataset and manually annotate the generations produced by GPT-4. When examining earlier works on factuality evaluation with crowd-workers, we found extensive errors and thus chose to annotate the data ourselves using a 4-way label (Factual, Subjective, Unverifiable, and False), where factuality judgments were verified using Google. We considered Factual and Subjective as entailed, and others as not entailed. All factuality annotations were done before running our experiments.

\subsection{Results}
In this section, we first verify that Algorithm~\ref{algo:conformally_factual_algo} indeed achieves the factuality guarantees of Theorem~\ref{th:CpEquiv}, and then we assess the utility of our outputs both quantitatively and qualitatively.

\subsubsection{Empirical factuality} \label{subsec:empirical_factuality}
\begin{figure}[!h]
  \centering\includegraphics[width=0.45\textwidth]{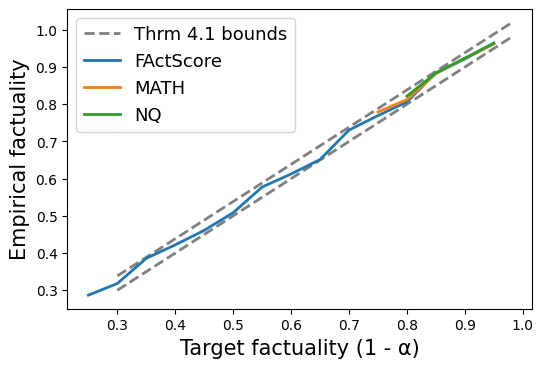}
  \caption{Target vs. empirical factuality. Each solid line starts at the base factuality of GPT-4 on the associated dataset. NQ and MATH overlap on the top right.}
  \label{fig:calib}
\end{figure}

Our main result, Theorem~\ref{th:CpEquiv}, states that we should attain roughly $1-\alpha$ factuality. To check that this happens in practice, we randomly split our datasets in half $1000$ times, fitting a threshold on the first half and measuring the factuality on the second half. This is done with $\SL_t$ implemented via sub-claims with frequency scoring. We plot the results in Figure~\ref{fig:calib}, and defer similar plots with the remaining scoring function to Appendix~\ref{app:calib}. These results show very tight control over factuality, which Theorem~\ref{th:CpEquiv} guarantees marginally over both the draw of a calibration set and test point. Going beyond standard conformal guarantees, if we require additional high probability guarantees over the calibration set, we find that there is additional variation in the empirical factuality with standard deviation $ \sim\!\! 0.09$, which we expect to decrease to zero as the calibration set grows large.

\subsubsection{Utility}

\begin{figure*}[!t]
  \centering
  {\includegraphics[width=0.28\textwidth]{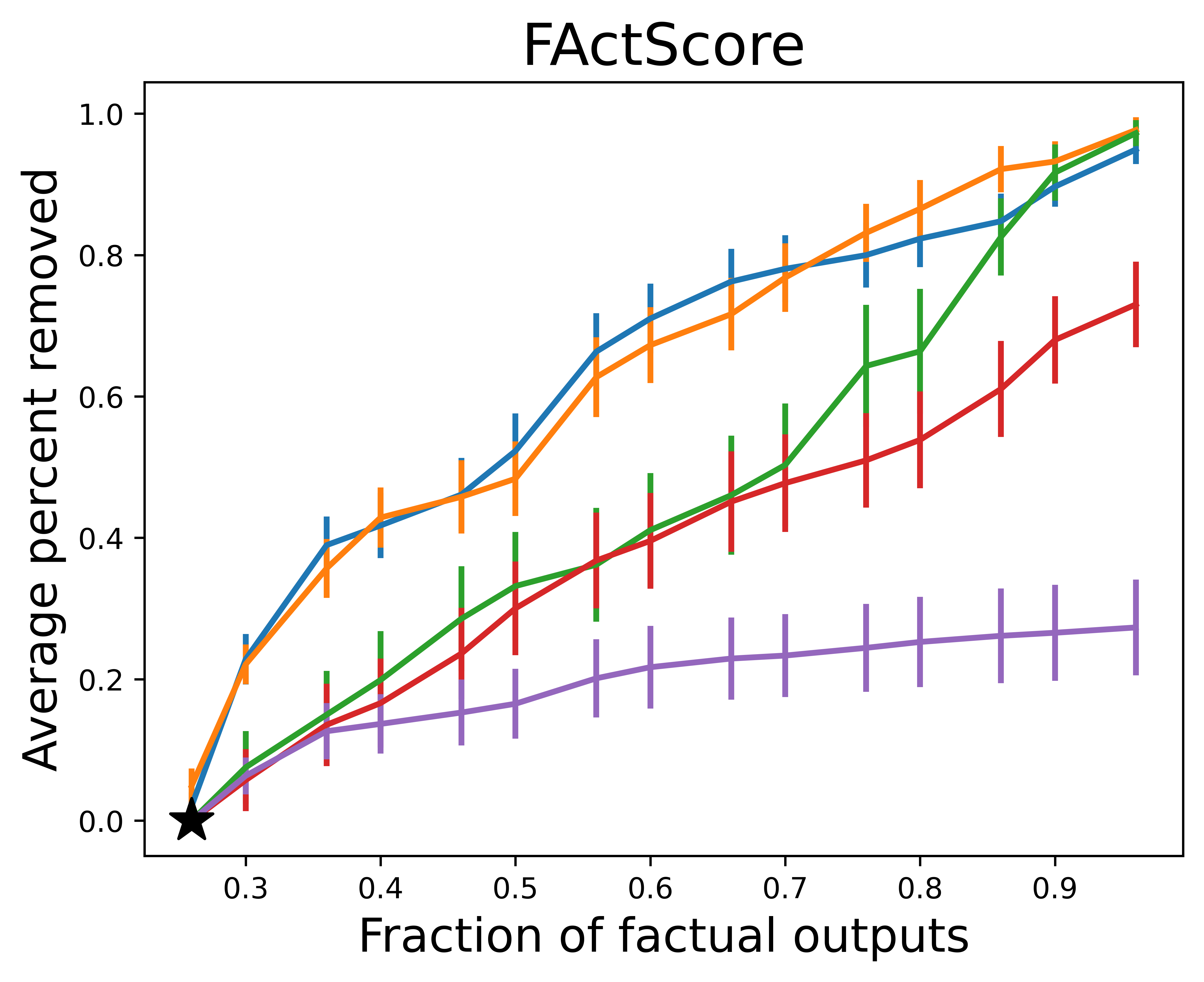}}
  {\includegraphics[width=0.28\textwidth]{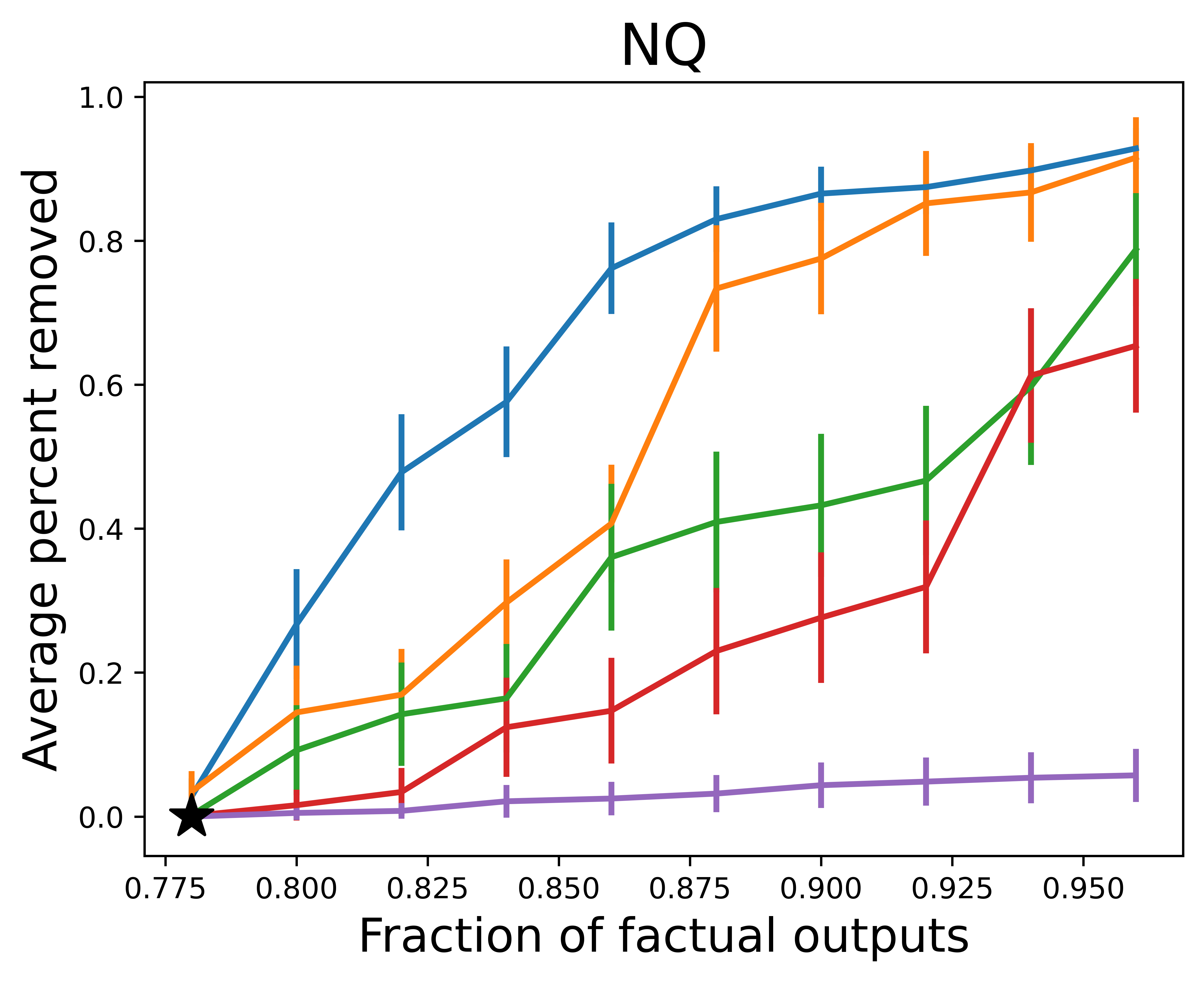}}
  {\includegraphics[width=0.4\textwidth]{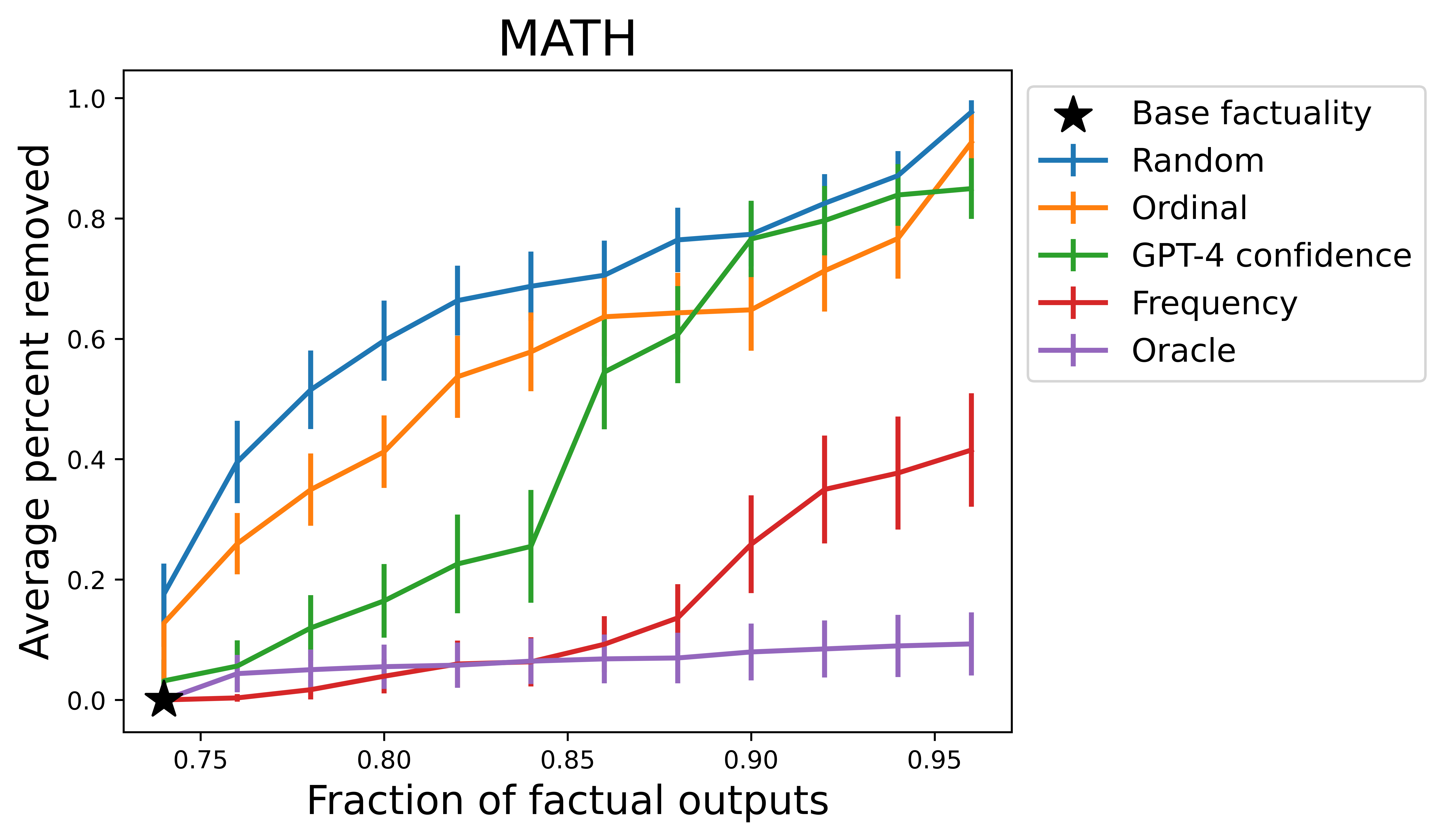}}
  \caption{Factuality vs. percent of sub-claims removed across all datasets. Frequency scoring (red) can lead to significant ($20-50\%$) gains in correctness while retaining the majority of claims when compared to the base GPT-4 model (star). The tick marks correspond to different values of target $\alpha$, and the standard deviations represent standard error.}
  \label{fig:factuality_vs_avg_removed}
\end{figure*}

The utility of the outputs crucially depends on the quality of the sub-claim scoring function: if a scoring function removes many correct sub-claims before the incorrect ones, the resulting output may be uninformative. We now evaluate our scoring functions to identify whether any of them achieve acceptable utility at various factuality levels. 

We start by plotting the percent of sub-claims that are removed under a leave-out-out evaluation and varying the confidence level (Figure~\ref{fig:factuality_vs_avg_removed}). 

First, the star on the bottom left of each plot represents the base performance of GPT-4, which has relatively low correctness across all tasks. For FActScore, this is very low and the large majority of outputs had at least one sub-claim that we deemed non-factual. For the remaining two datasets, the base performance of GPT-4 is higher at around $75\%$, but still too low for many high-stakes settings. 

Second, Figure~\ref{fig:factuality_vs_avg_removed} suggests that our two scoring functions derived from an LM --- frequency scoring and GPT-4 confidence scoring --- provide reasonable tradeoffs between utility and factuality. For frequency scoring, we observe particularly promising performance: factuality on FActScore can be increased from $\sim \!\!25\%$ to $\sim \!\!80\%$ while keeping around half of the original sub-claims, on NQ we can increase factuality by $\sim \!\!15\%$ by only removing about one-quarter of sub-claims, and on MATH we can similarly increase factuality by $\sim \!\!15\%$ but by only removing $\sim \!\!10\%$ of sub-claims. 

\begin{figure}[!h]
  \centering
  {\includegraphics[width=0.42\textwidth]{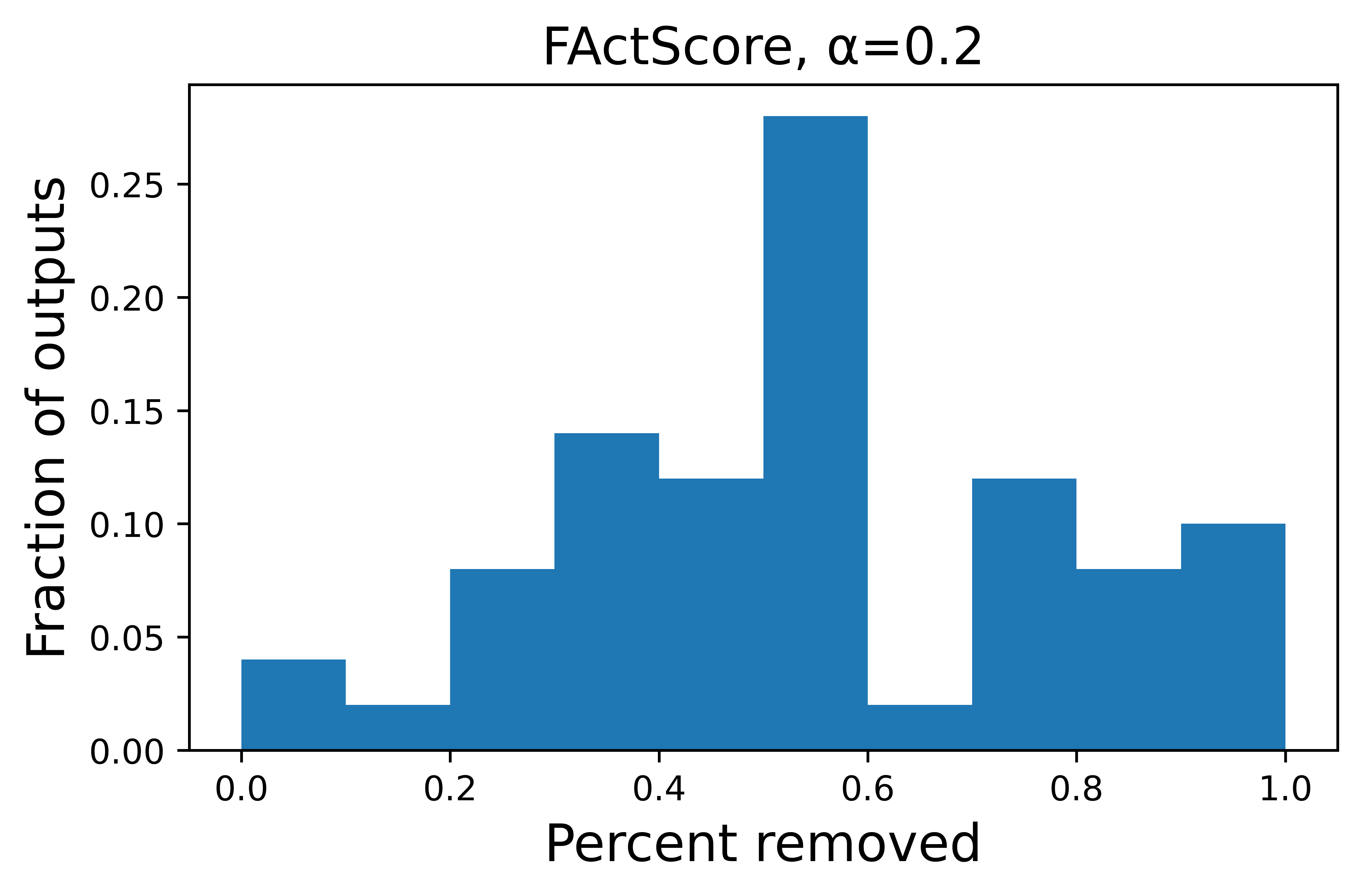}}
  \caption{Histogram of percent of sub-claims removed for frequency scoring and $\alpha=0.2$ on FActScore. }
  \label{fig:histograms}
\end{figure}

Taking a closer look at individual points on the plots of Figure~\ref{fig:factuality_vs_avg_removed}, we show that the behavior of our method is far from simply abstaining on difficult examples. Analyzing the $\alpha=0.2$ model for FActScore, we find that most examples have $\sim \!\!50\%$ of their claims removed, leading to a less detailed but still informative responses in most cases. The standard deviations for the percent of sub-claims removed can be made very low (at the cost of slightly worse means) with $\emph{ranking-based}$ scoring functions, which we define and present results for in Appendix~\ref{app:ranking}. 

Finally, we complement our quantitative analysis with samples generated by the conformal factuality model with frequency scoring. Our main dataset of interest is FActScore, as GPT-4 hallucinates aggressively on this task. We thus choose the $\alpha=0.2$ threshold and present examples in Table~\ref{tab:io_examples}, with further examples in Appendix~\ref{app:more-io}. In all cases shown here, we find that the original GPT-4 output contained a significant fraction of falsehoods, which was successfully removed with conformal factuality.

\begin{table*}[!ht]
  \centering
  \begin{tabular}{p{8.5cm} p{6.5cm}}
    \small{\textbf{Original}} & \small{\textbf{Conformal Factuality}}  \\
    \scriptsize{\texttt{Zamfir Ralli-Arbore (1848-1933) was a Romanian political activist and historian from Bessarabia, who spent much of his life in exile. \hl{As a member of the National Liberal Party, he campaigned for the union of his native region with the Kingdom of Romania}, and was a prominent opponent of Russian and Soviet policies. He was also a noted historian, \hl{specializing in the history of the Moldavia and Wallachia during the Middle Ages.}}} & \scriptsize{\texttt{Zamfir Ralli-Arbore, born in 1848, was a Romanian political activist from Bessarabia. He passed away in 1933.}} \\
    \hline\ignore{
    \small{\textbf{Rennie Fritchie:}}\\
    \scriptsize{\texttt{Rennie Fritchie, Baroness Fritchie, DBE (born 1942) is a \hl{British public servant} and crossbench member of the House of Lords. She has held a number of high-profile roles in public service, including serving as the Commissioner for Public Appointments in the United Kingdom from 1999 to 2005. \hl{She was also the Chair of Nominet UK, the .uk domain name registry, from 2005 to 2010.} In 2005, she was made a life peer as Baroness Fritchie, of Gloucester in the County of Gloucestershire. \hl{She has also served on the boards of several other organizations, including the BBC and the Prince's Trust.}}
    } & \scriptsize{\texttt{Rennie Fritchie, Baroness Fritchie, DBE, is a \hl{British public servant} who was born in 1942. She has held a number of high-profile roles in public service, including serving as the Commissioner for Public Appointments in the United Kingdom from 1999 to 2005. In 2005, she was made a life peer as Baroness Fritchie, of Gloucester in the County of Gloucestershire. She is also a crossbench member of the House of Lords.
}} \\
    \hline}
    \scriptsize{\texttt{John Estes may refer to: 1. Sleepy John Estes (1899-1977), American blues guitarist, songwriter and vocalist. \hl{2. John Adam Estes (1919-2001), American politician who served as a member of the Texas House of Representatives.} 3. John Estes (American football) \hl{(born 1986)}, American football center. \hl{4. John C. Estes, American screenwriter known for his work on films like "Mean Creek" and "The Details".} Without more specific information, it's hard to provide details on the correct John Estes.}
    } & \scriptsize{\texttt{Sleepy John Estes was an American blues guitarist, songwriter and vocalist. He was born in 1899 and died in 1977.
}} \\
    \hline
    \scriptsize{\texttt{Michael Valpy is a Canadian journalist, author, and lecturer. He was born on \hl{July 30}, 1942 in Toronto, Ontario. He is best known for his work as a senior writer for The Globe and Mail, a leading Canadian newspaper, where he covered both national and international news. Valpy has also written several books on Canadian politics and history, and has taught journalism at several universities. \hl{He is a three-time nominee for the Governor General's Award}, Canada's top literary honor.}
    } & \scriptsize{\texttt{Michael Valpy is a Canadian journalist, author, and lecturer. He is best known for his work as a senior writer for The Globe and Mail, a leading Canadian newspaper. In his career, Valpy covered both national and international news for the newspaper. In addition to his journalism career, he has written several books on Canadian politics and history. Valpy has also taught journalism at several universities.
}} \\
\ignore{   \hline
    \scriptsize{\texttt{The first Prime Minister of Nigeria was Sir Abubakar Tafawa Balewa, \hl{serving from 1957 until 1966.}}
    } & \scriptsize{\texttt{The first ministry head of state in Nigeria was Sir Abubakar Tafawa Balewa, who served as the country's Prime Minister.
}}\\}
   \hline
    \scriptsize{\texttt{We have $i^5=i^4\cdot i=i$, $i^{-25}=$\hl{$i^{-24}\cdot i^{-1}=i$}, and $i^{45}=i^{44}\cdot i=i$. Therefore, \hl{$i^5+i^{-25}+i^{45}=\boxed{3i}$}}
    } & \scriptsize{\texttt{The given steps are:$i^5=i^4\cdot i=i$\textcolor{gray}{\textbackslash n\textbackslash n}$i^{45}=i^{44}\cdot i=i$\textcolor{gray}{\textbackslash n\textbackslash n}Using these steps, we can simplify the given expression:$i^5+i^{-25}+i^{45}=i+i^{-25}+i=i+i^{-25}+i$  
}} \\
  \end{tabular}
  \caption{Comparison of original output by GPT-4 versus our modified output from conformal factuality with frequency scoring. Incorrect claims are highlighted in \hl{red}. The first three examples are from FActScore and use $\alpha=0.2$. The last sample from MATH, where the question asked for the evaluation of the expression $i^5+i^{-25}+i^{45}$, uses $\alpha=0.1$.}
  \label{tab:io_examples}
\end{table*}

In addition to the FActScore outputs, we also show an example from MATH in Table ~\ref{tab:io_examples} (with $\alpha=0.1$) which shows how our method can be successfully applied to reasoning tasks. Here, the math problem required evaluating three expressions and then taking their sum, but GPT-4 was not able to correctly evaluate one of the expressions and thus provided an incorrect final answer. The conformal factuality model identified and removed the incorrect claim, leading to a partial proof that was presented to the user with a single remaining step left for the user.

\ignore{Note that all plotted lines intersect the origin, meaning that \emph{none of the original model outputs were factual}. For several of the sub-claim scoring functions, we are able to retain a non-trivial amount of the sub-claims while guaranteeing a reasonable fraction of fully factual outputs. Across both ranking and non-ranking scoring functions, frequency+GPT-4 scoring seems to perform best for small $\alpha$ (high factuality). The standard deviations for non-ranking-based scoring functions are much larger because they remove a much more variable number of sub-claims, but the respective means are generally smaller. }
\ignore{
\begin{figure}
    $\sfL(x) =$  `Lanny Flaherty is an American actor known for his roles in a variety of films and television shows. Born on July 27, 1942 in Pontotoc, Mississippi, he has been active in the entertainment industry since the 1970s. \hl{Some of his most notable roles include Sheriff Tatum in ``Natural Born Killers," Earl in ``Water for Elephants," and Deputy Clinton Pell in ``Mississippi Burning." Flaherty has also made appearances in popular television series such as ``Boardwalk Empire" and ``Law \& Order."} {Despite his extensive filmography, he has managed to maintain a relatively low profile in Hollywood.}' \\
    $A_{\hat q_\alpha}(x)$  = `{Lanny Flaherty is an American actor}', `{He is known for his roles in a variety of films and television shows}', '{Flaherty was born on July 27, 1942}', '{He was born in Pontotoc, Mississippi}'\\
    $M(A_{\hat q_\alpha}(x))$ = `{Lanny Flaherty is an American actor known for his roles in a variety of films and television shows. He was born on July 27, 1942 in Pontotoc, Mississippi.}' 
    \caption{Full example with $\alpha=0.3$ and frequency scoring. $\sfL(x)$ denotes the original model output, $A_{\hat q_\alpha}$ denotes the accepted sub-claims, and $M(A_{\hat q_\alpha})$ denotes the merged accepted sub-claims. The red highlight denotes the non-factual text. After applying our method, the output sequence is factual.}
    \label{fig:io_example}
\end{figure}
}

\section{Conclusion}
We described conformal factuality, a framework that connects conformal prediction and language modeling, and importantly leads to a practical algorithm for obtaining conformal factuality guarantees on LM outputs. We gave a natural implementation with sub-claim scoring functions and showed that we can indeed get factual and useful outputs on both question-answering and reasoning datasets. Our work still has limitations, including the guarantees being restricted to a prescripted distribution $P$, and the fact that our bounds are marginal with respect to the draw of calibration and test data. 
But our work is the first step to enabling the application of many sophisticated conformal prediction algorithms in improving the outputs of language models -- including those that address challenges such as distribution shifts.

\section*{Limitations} 
Our work depends on standard split conformal prediction, which has important limitations in practice. First, the coverage guarantee is not conditional, meaning that we do not have a guarantee on the conformal set associated with every input, but instead, we have coverage on average over inputs.  This means that when the calibration set is defined across multiple tasks or users, the factuality of any one domain or user may differ significantly from the target factuality level.
Second, the conformal guarantee is also marginal over the draw of the calibration set. This means that in settings where one has a small, fixed calibration dataset and repeatedly uses the same threshold, the coverage associated with this threshold may deviate from the target coverage. Lastly, in real-world scenarios where distributions change, the threshold computed on past calibration data can fail to maintain the desired coverage. In our case, this means that the factuality guarantee of our language model may be lost if, for example, the distribution of inputs were to change drastically.  

While these drawbacks of conformal prediction exist, we believe this work is a step in the right direction toward guaranteeing the factuality of language models. Several works build on the framework of split conformal prediction to tackle the challenges mentioned, and the connections established in this work could enable the use of those approaches \citep{gibbs2023conformal, ding2023classconditional, barber2023conformal, tibshirani2020conformal, vovk2012conditional}. Moreover, we observe promising experimental results that show how we can effectively remove hallucinations from language model outputs, and we expect this to improve with better uncertainty quantification methods.

\section*{Acknowledgements}
We thank Zitong Yang for comments related to an earlier version of this manuscript, as well as Neil Band and John Duchi for helpful discussions. TH was supported by a gift by Open Philantropy, the Tianqiao and Chrissy Chen Foundation, and a grant by Samsung GRO.

\newpage
\newpage
\newpage

\bibliographystyle{apalike} 
\bibliography{main,all}

\newpage
\appendix
\onecolumn

\ignore{\renewcommand{\contentsname}{Contents of Appendix}
\tableofcontents
\addtocontents{toc}{\protect\setcounter{tocdepth}{3}} 
\clearpage}

\ignore{with:
\begin{align*}
    \texttt{Please breakdown the following into independent facts:} \text{\{L(x)\}}. 
\end{align*}

\begin{figure}[ht]
  \centering
  \begin{subfigure}
    \centering
    \includegraphics[width=\linewidth]{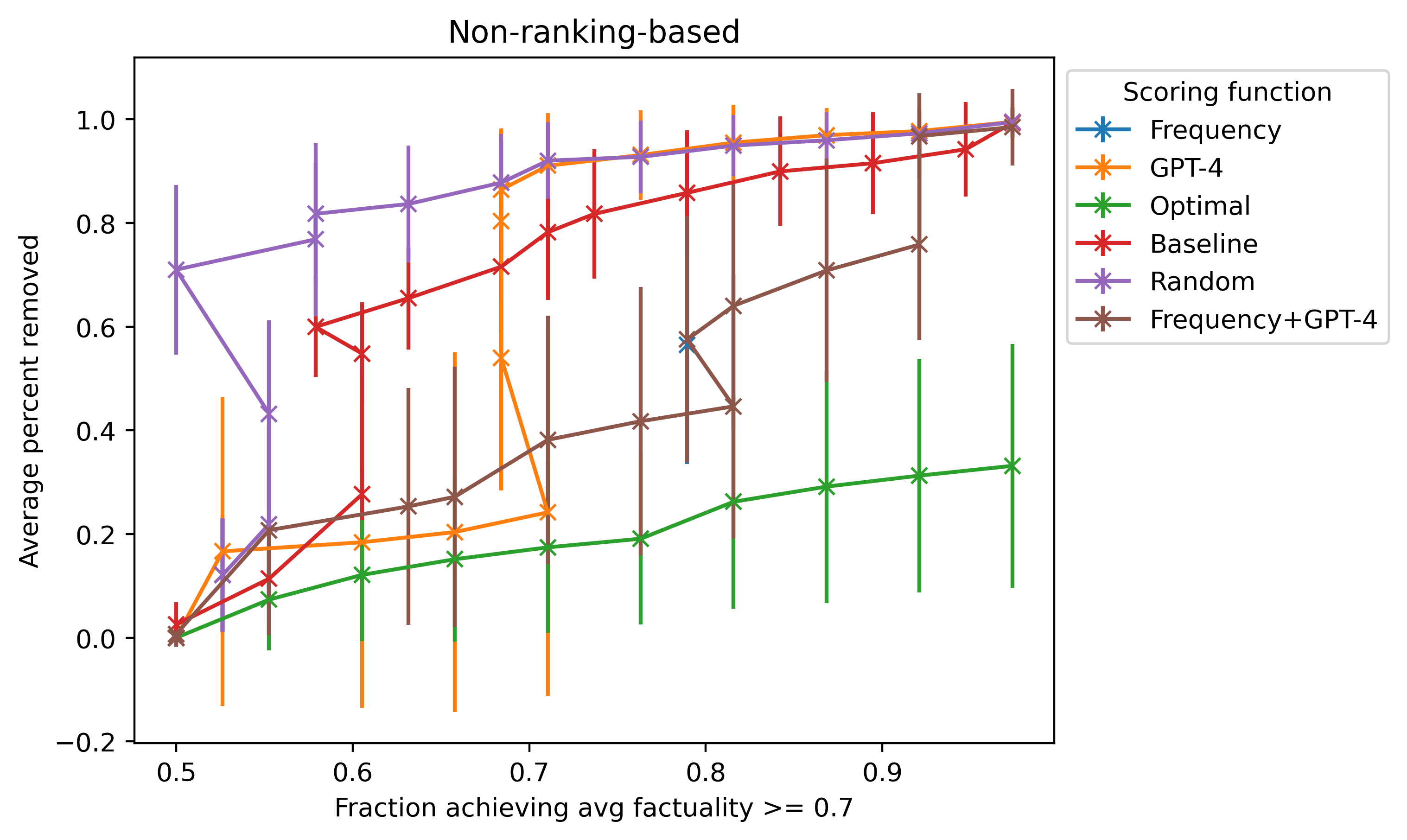}
  \end{subfigure}
  \begin{subfigure}
    \centering
    \includegraphics[width=\linewidth]{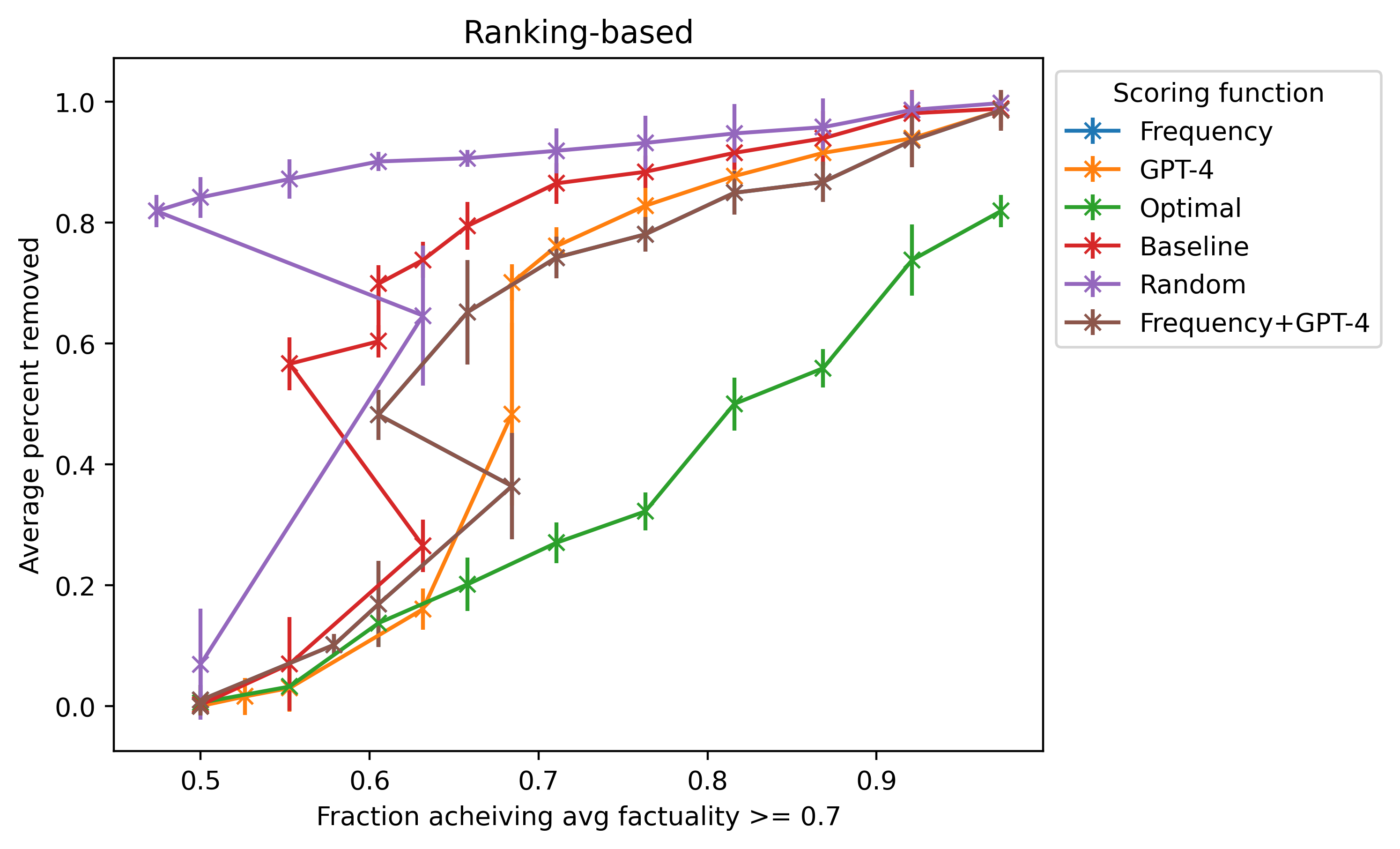}
  \end{subfigure}
  \caption{$\SL_t$ via sub-claims with $a=0.7$.}
  \label{fig:partial}
\end{figure}

}

\section{Prompts used in experiments}
\label{app:prompts}

We use prompts to implement both the sub-claim separator $\sfS$ and merger function $\sfM$ (defined in Section~\ref{sec:implementation}). The only other prompt we use is for frequency scoring. For convenience, we perform GPT-4 confidence scoring alongside sub-claim separation. All of these prompts appear in Table~\ref{tab:prompts}. We use the \texttt{`gpt-4'} endpoint, set \texttt{max\_tokens} to $1000$ and \texttt{temperature} to $0.0$. We used GPT-4 with these prompts between December $15$ and January $15$.

\begin{table*}[ht!]
  \centering
  \begin{tabular}{p{15cm} }
    \small{\textbf{Separator (for all datasets)/GPT-4 confidence scoring}} \\
    \scriptsize{\texttt{Please breakdown the following input into a set of small, independent claims (make sure not to add any information), and return the output as a jsonl, where each line is {subclaim:[CLAIM], gpt-score:[CONF]}.\ignore{\n} The confidence score [CONF] should represent your confidence in the claim, where a 1 is obvious facts and results like `The earth is round' and `1+1=2'. A 0 is for claims that are very obscure or difficult for anyone to know, like the birthdays of non-notable people. If the input is short, it is fine to only return 1 claim. The input is: }}  \\
    \small{\textbf{Merger (for FActScore)}} \\
    \scriptsize{\texttt{You will get an instruction and a set of facts that are true. Construct an answer using ONLY the facts provided, and try to use all facts as long as its possible. If no facts are given, reply to the instruction incorporating the fact that you dont know enough to fully respond. \textcolor{gray}{\textbackslash n\textbackslash n}The facts:\textcolor{gray}{\textbackslash n} \{claim\_string\}\textcolor{gray}{\textbackslash n\textbackslash n}The instruction:\textcolor{gray}{\textbackslash n}\{prompt\}}} \\
    \small{\textbf{Merger (for NQ)}} \\
    \scriptsize{\texttt{You will get a natural question and parts of an answer, which you are to merge into coherent prose. Make sure to include all the parts in the answer. There may be parts that are seemingly unrelated to the others, but DO NOT add additional information or reasoning to merge them. \textcolor{gray}{\textbackslash n\textbackslash n}The parts:\textcolor{gray}{\textbackslash n}\{claim\_string\}\textcolor{gray}{\textbackslash n\textbackslash n}The question:\textcolor{gray}{\textbackslash n}\{prompt\}. Remember, DO NOT add any additional information or commentary, just combine the parts.}} \\
    \small{\textbf{Merger (for MATH)}} \\
    \scriptsize{\texttt{"You will get a math problem and a set of steps that are true. Construct an answer using ONLY the steps provided. Make sure to include all the steps in the answer, and do not add any additional steps or reasoning. These steps may not fully solve the problem, but merging them could assist someone in solving the problem. \textcolor{gray}{\textbackslash n\textbackslash n}nThe steps:\textcolor{gray}{\textbackslash n}\{claim\_string\}\textcolor{gray}{\textbackslash n\textbackslash n}The math problem:\textcolor{gray}{\textbackslash n}\{prompt\}. Remember, do not do any additional reasoning, just combine the given steps.}}  \\
    \small{\textbf{Frequency scoring}} \\
    \scriptsize{\texttt{You will get a list of claims and piece of text. For each claim, score whether the text supports, contradicts, or is unrelated to the claim. Directly return a jsonl, where each line is \{"id":[CLAIM\_ID], "score":[SCORE]\}. Directly return the jsonl with no explanation or other formatting. For the [SCORE], return 1 for supports, -1 for contradicts, and 0 for unrelated. The claims are:\textcolor{gray}{\textbackslash n}\{claim\_string\}\textcolor{gray}{\textbackslash n\textbackslash n}The text is:\textcolor{gray}{\textbackslash n}\{output\}
}}  \\
\\
  \end{tabular}
  \caption{Prompts for sub-claim separator $\sfS$ and merger function $\sfM$, as well as LM-based sub-claim scoring functions (GPT-4 confidence and frequency). The prompt for frequency scoring is used to evaluate 5 alternate output sequences generated by GPT-4 with temperature $1.0$ .} 
  \label{tab:prompts}
\end{table*}

\section{Partial factuality}
\label{app:partial-factuality}

Here we give experimental results for the partial entailment/factuality setting of Section~\ref{subsec:partial_entailment}, which relaxes the standard formulation with an acceptable fraction of entailed sub-claims $a \in [0, 1]$. We choose $a=0.7$ for FActScore, $a=0.9$ for NQ, and $a=0.8$ for MATH, and plot factuality vs. percent of sub-claims removed in Figure~\ref{fig:partial_factuality_vs_avg_removed}. While the condition to be considered factual is relaxed, we observe a lower percentage of sub-claims removed.

\begin{figure*}[!h]
  \centering
  {\includegraphics[width=0.28\textwidth]{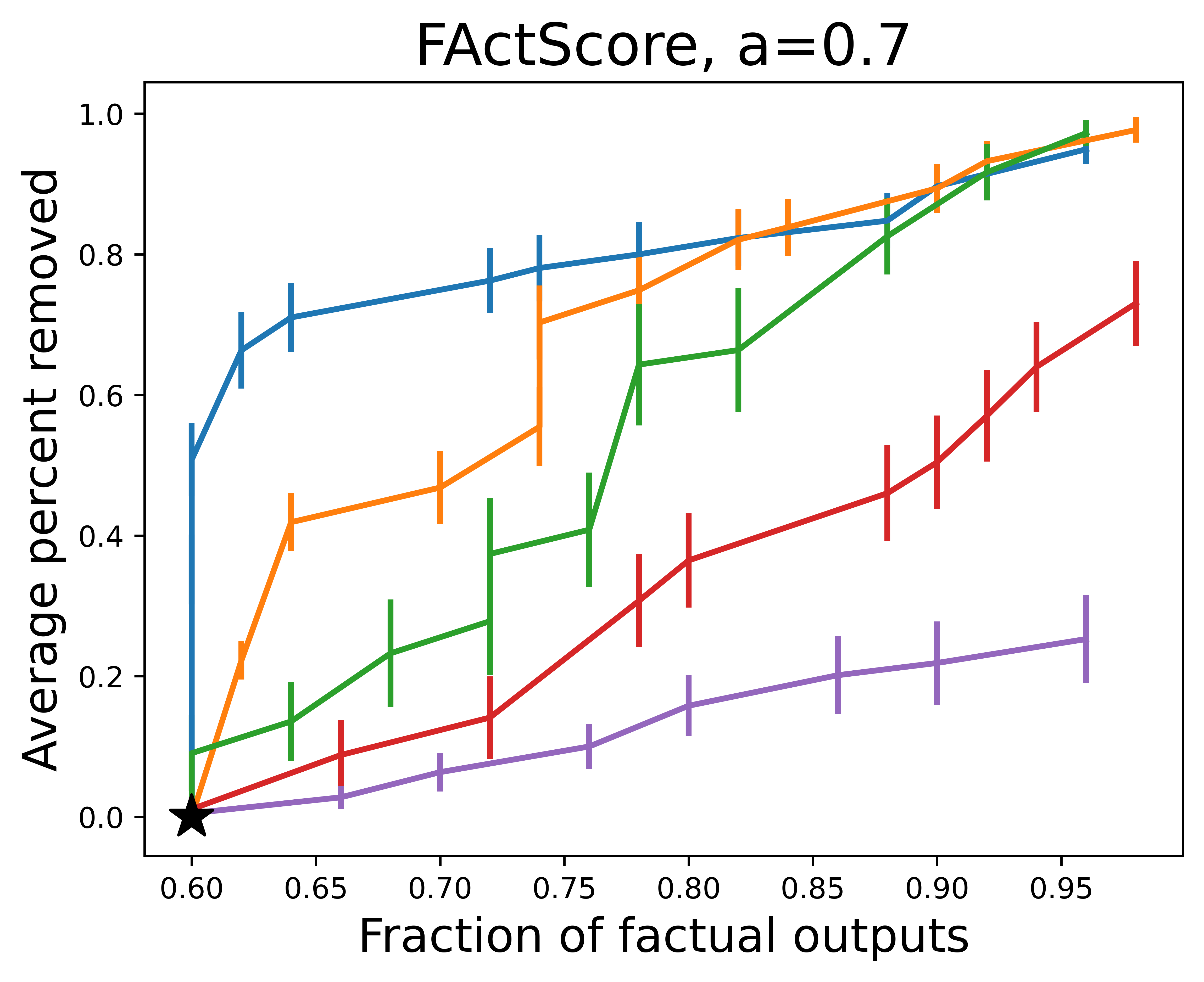}}
  {\includegraphics[width=0.28\textwidth]{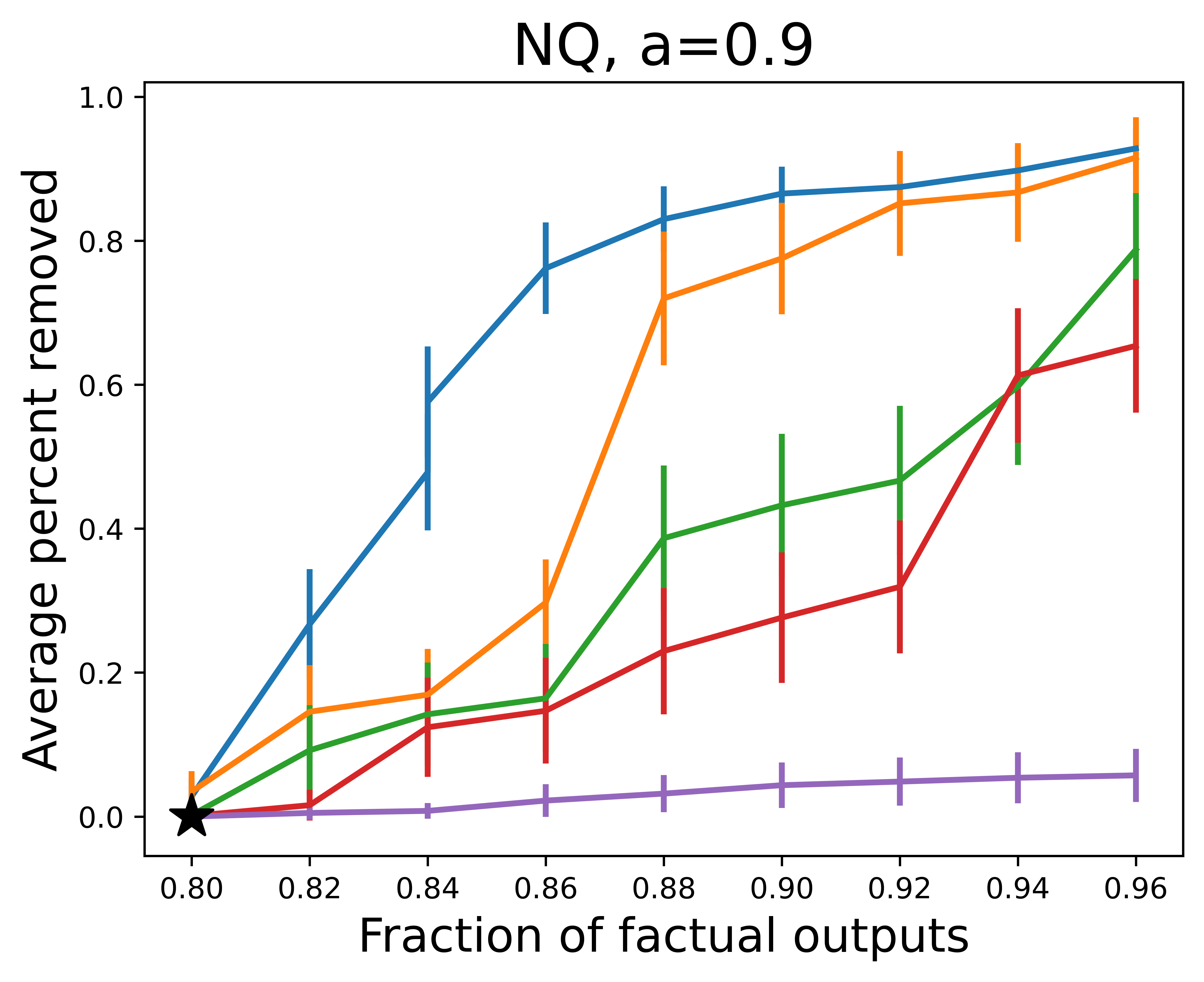}}
  {\includegraphics[width=0.4\textwidth]{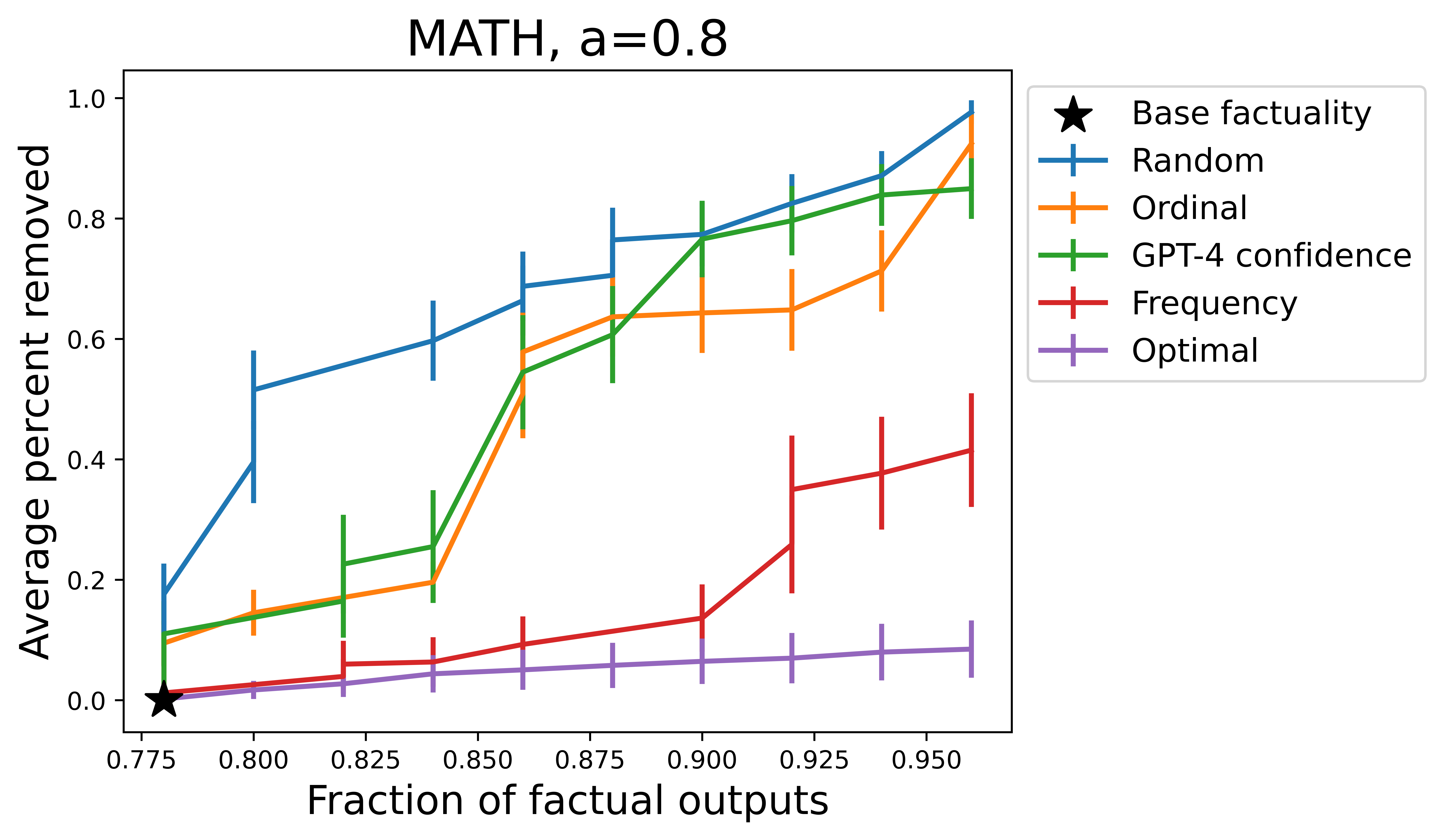}}
  \caption{Factuality vs. percent of sub-claims removed across all datasets for partial factuality setting. The tick marks correspond to different values of target $\alpha$, and the standard deviations represent standard error.}
  \label{fig:partial_factuality_vs_avg_removed}
\end{figure*}

\section{Ranking-based scoring functions}
\label{app:ranking}

Any sub-claim scoring function can be transformed into a \emph{ranking-based} scoring function. These aim to remove a fixed percentage of facts, which may be desirable in cases where modifying the outputs in roughly the same way is desirable. Here, for a sub-claim scoring function $s\colon 2^\sY \times \sY \mapsto \Rset$, we define a ranking function $f_s \colon 2^\sY \times \sY \mapsto \Nset$, where $f(\{c_i\}_{i=1}^n, c_j)$ returns the rank of $s(\{c_i\}_{i=1}^n, c_j)$ among the set $\{s(\{c_i\}_{i=1}^n, c_k)\}_{k=1}^n$ in increasing order. Then, for each of sub-claim scoring functions $s$ defined above, we apply the score:  
\begin{align*}
     \frac{f_s(\{c_i\}_{i=1}^n, c_j)}{n}.
\end{align*}
We divide by $n$ to account for varying-length output sequences. In Figure~\ref{fig:ranking_factuality_vs_avg_removed}, we plot factuality vs. percent of sub-claims removed for the ranking-based versions of all our sub-claim scoring functions.  These do remove a higher percent of sub-claims, but the standard deviations are smaller.

\begin{figure*}[!ht]
  \centering
  {\includegraphics[width=0.28\textwidth]{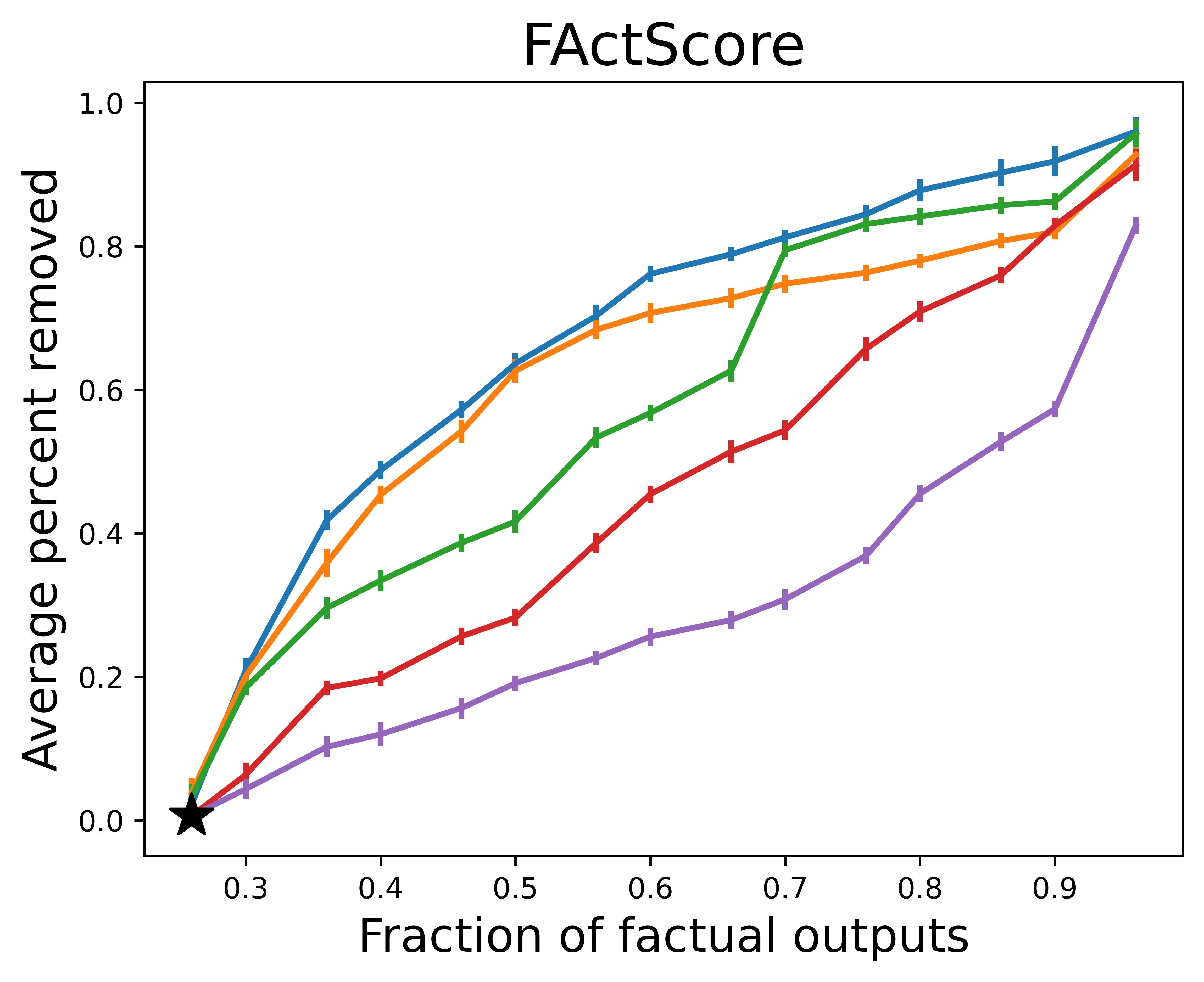}}
  {\includegraphics[width=0.28\textwidth]{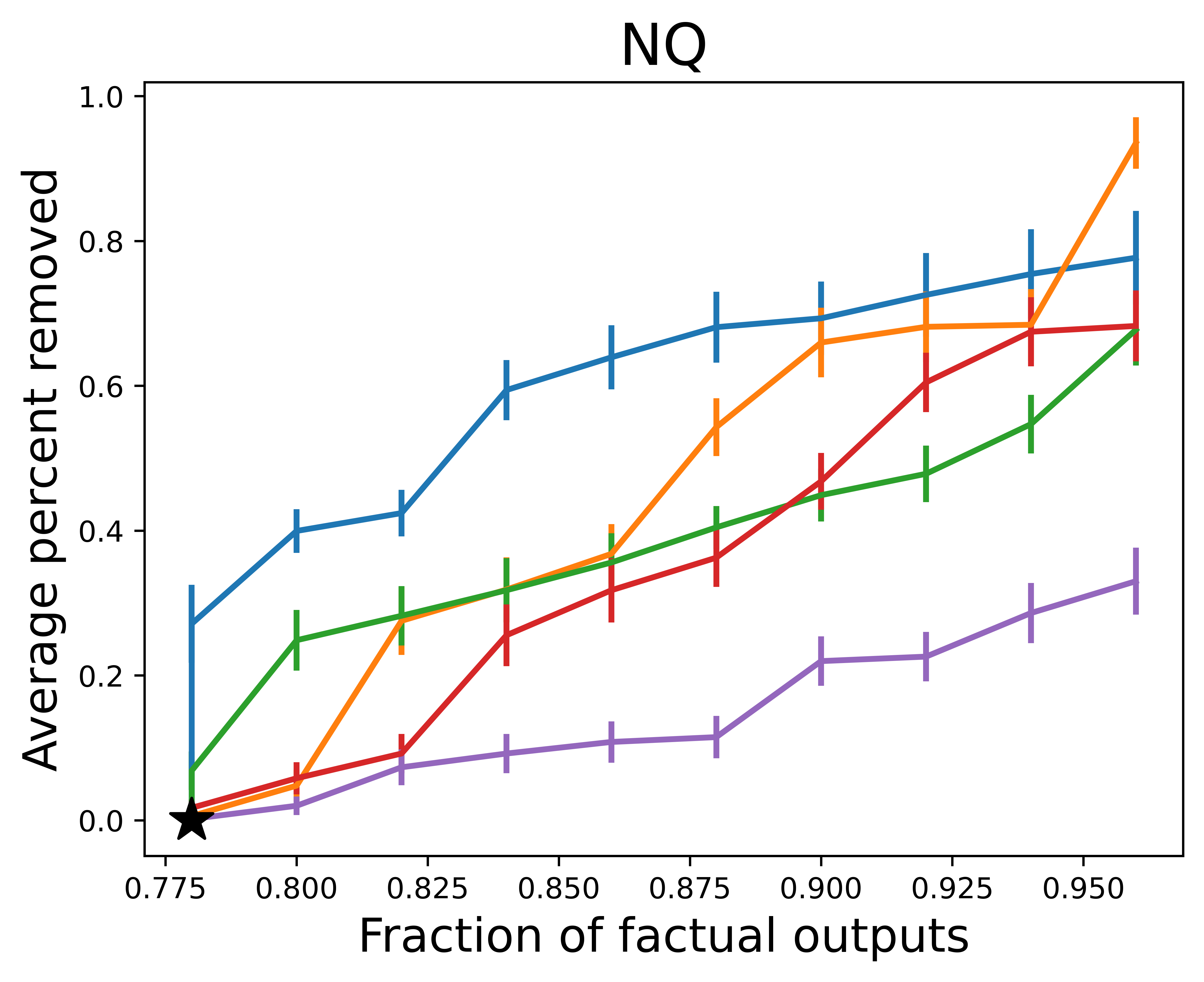}}
  {\includegraphics[width=0.4\textwidth]{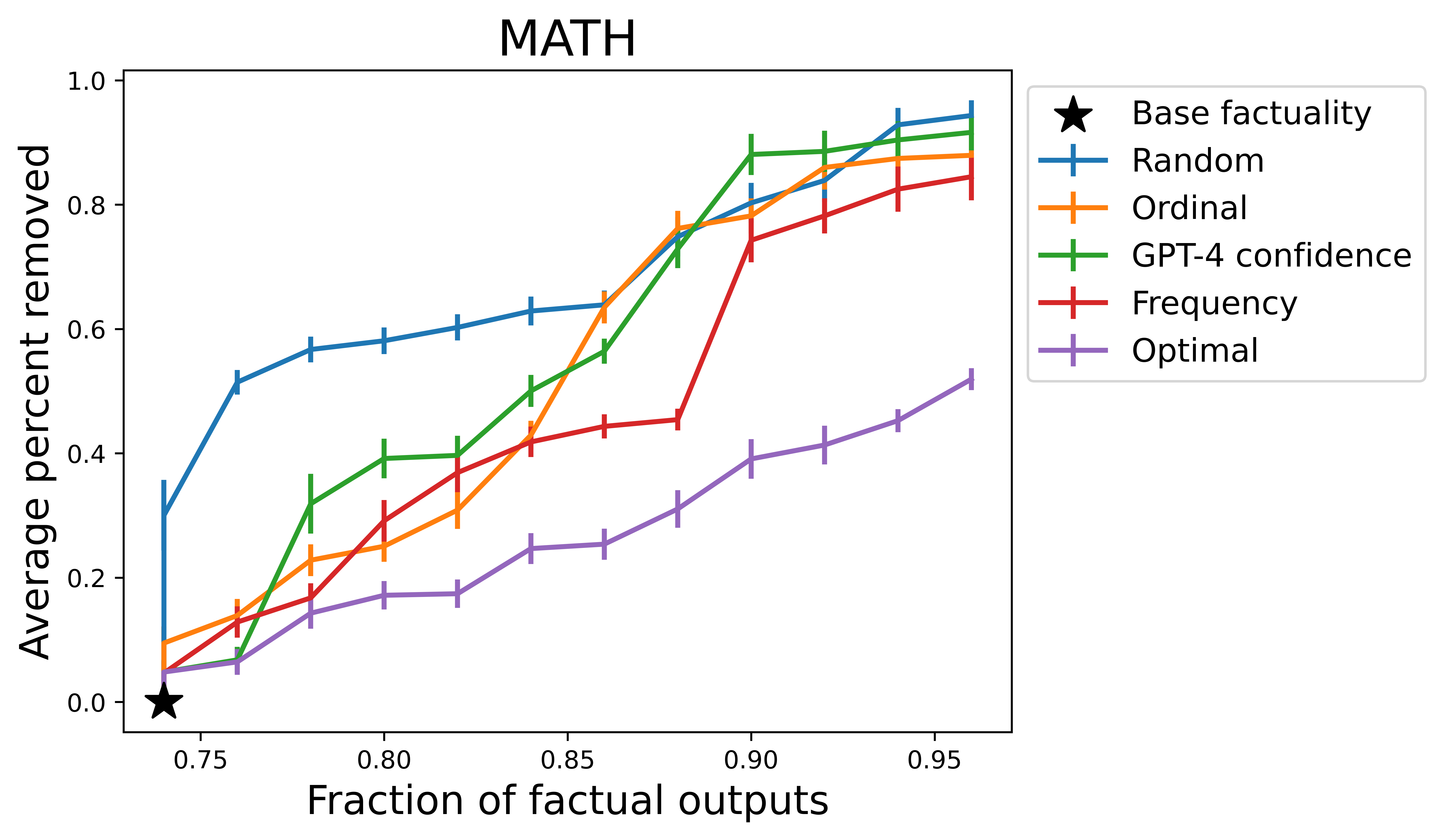}}
  \caption{Factuality vs. percent of sub-claims removed across all datasets with \emph{ranking-based} scoring functions. The tick marks correspond to different values of target $\alpha$, and the standard deviations represent standard error.}
  \label{fig:ranking_factuality_vs_avg_removed}
\end{figure*}

\section{More conformal factuality output examples}
\label{app:more-io}

Here we give more conformally factual output examples, supplementing those in Table~\ref{tab:io_examples}. We give examples for FActScore in Table~\ref{tab:extra_factscore}, NQ in Table~\ref{tab:extra_io_nq}, and MATH in Table~\ref{tab:extra_math}. 

\begin{table*}[ht!]
  \centering
  \begin{tabular}{p{9cm} p{7cm}}
    \small{\textbf{Original}} & \small{\textbf{Conformal Factuality}}  \\
    \scriptsize{\texttt{\hl{Luis Guillermo Rivera is a professional athlete from Colombia who specializes in high jump. He was born on March 20, 1989 in Apartado, Colombia. Rivera has represented Colombia in several international competitions, including the World Championships in Athletics and the Pan American Games. He won the gold medal in the high jump at the 2013 South American Championships in Athletics. His personal best in the high jump is 2.31 meters, which he achieved in 2013.}
}}   & \scriptsize{\texttt{I'm uncertain about the specific details, and I recommend using a search engine like Google for more accurate and up-to-date information.}} \\
    \scriptsize{\texttt{Roxana Diaz is a \hl{Venezuelan} athlete \hl{who specializes in the 400 meters hurdles. She represented her country in various international competitions, including the World Championships and the Pan American Games.} Diaz is known for her speed, agility, and determination on the track. \hl{She has won several medals throughout her career, making her one of the most successful athletes in her country.} }}   & \scriptsize{\texttt{Roxana Diaz is an athlete \hl{who has represented Venezuela in various international competitions.} She has also participated in the World Championships.
}} \\
    \scriptsize{\texttt{Sharad Kumar is an Indian para-athlete who competes in high jump events. He was born on March 1, 1992, in Bihar, India. \hl{Kumar was diagnosed with polio at a young age, which affected his left leg.} Despite this, he pursued his passion for sports and began training as a high jumper.\textcolor{gray}{\textbackslash n\textbackslash n}Kumar has represented India in several international competitions and has won numerous medals. \hl{He won a gold medal at the 2014 Asian Para Games in Incheon, South Korea, and a silver medal at the 2018 Asian Para Games in Jakarta, Indonesia. He also won a bronze medal at the 2017 World Para Athletics Championships in London.}\textbackslash nnIn 2021, Kumar won a bronze medal in the men's high jump T42 category at the Tokyo Paralympics, marking his first Paralympic medal. \hl{His personal best in high jump is 1.90 meters, which he achieved at the 2017 World Para Athletics Championships.}}}   & \scriptsize{\texttt{Sharad Kumar is an Indian para-athlete who was born on March 1, 1992, in Bihar, India. He competes in high jump events and despite his condition, he pursued his passion for sports and began training as a high jumper. Sharad Kumar has represented India in several international competitions and has won numerous medals in his career. Most notably, in 2021, he won a bronze medal in the men's high jump T42 category at the Tokyo Paralympics.
}} \\
    \scriptsize{\texttt{Katsunosuke Hori is a Japanese actor and voice actor. \hl{He was born on November 28, 1941} in Tokyo, Japan. He is best known for his work in anime and video games, \hl{including roles in "Neon Genesis Evangelion", "Ghost in the Shell: Stand Alone Complex", and "Final Fantasy VII: Advent Children".} He is also known for his dubbing roles in foreign films, \hl{such as the voice of Gandalf in the Japanese version of "The Lord of the Rings" trilogy.}}}   & \scriptsize{\texttt{Katsunosuke Hori is a Japanese actor and voice actor. He is known for his work in anime and video games.}} \\
    \scriptsize{\texttt{Ra Jong-yil is a South Korean diplomat and academic. He served as the South Korean ambassador to the United Kingdom and Japan. He also served as the deputy director of the National Intelligence Service. As an academic, he has taught at various universities and written several books on North Korea and international relations.}}   & \scriptsize{\texttt{Ra Jong-yil is a South Korean diplomat and academic who has served as the South Korean ambassador to both the United Kingdom and Japan. In addition to his diplomatic work, he has written several books on North Korea and international relations.
}} \\
    \scriptsize{\texttt{Luis N. Rivera-Pag\u00e1n is a Puerto Rican theologian and author. \hl{He was the first Latino to be tenured as a professor at Princeton Theological Seminary,} where he taught for many years. \hl{His work focuses on the intersection of religion, colonialism, and social justice, particularly in the context of Latin America and the Caribbean.} He has written several books, including "A Violent Evangelism: The Political and Religious Conquest of the Americas" and "Essays from the Margins".
}}   & \scriptsize{\texttt{Luis N. Rivera-Pag\u00e1n is a Puerto Rican theologian and author who has written several books. His work primarily focuses on the intersection of religion, colonialism, and social justice. He has spent many years teaching at Princeton Theological Seminary."
}} \\
    \scriptsize{\texttt{Jesse Foppert is a former professional baseball player from the United States. He was born on July 10, 1980 in Reading, Pennsylvania. Foppert was a pitcher and played in Major League Baseball (MLB) for the San Francisco Giants and the Seattle Mariners. He was considered one of the top pitching prospects in baseball in 2003, but his career was derailed by injuries. After retiring from professional baseball, Foppert became a coach.}}   & \scriptsize{\texttt{Jesse Foppert is a former professional baseball player from the United States, born on July 10, 1980 in Reading, Pennsylvania. He was a pitcher in Major League Baseball (MLB), playing for the San Francisco Giants and the Seattle Mariners. Unfortunately, his career was derailed by injuries. After retiring from professional baseball, Foppert became a coach.
}} \\
    \scriptsize{\texttt{Zelma Wilson (\hl{1926}-1996) was an American architect \hl{known for her work in the field of sustainable design. She was one of the first women to establish her own architectural firm in California and was a pioneer in the use of solar energy in architecture. Wilson's designs often incorporated natural materials and were designed to blend in with their surroundings.} Some of her most notable projects include the Ojai City Hall and the \hl{Ventura County Government Center.} She was also a dedicated advocate for women in architecture and served as a mentor to many young architects.
}}   & \scriptsize{\texttt{Zelma Wilson was an American architect.}} \\
  \end{tabular}
  \caption{FActScore with $\alpha=0.2$ and frequency scoring: comparison of original output by GPT-4 versus our modified output from conformal factuality. Incorrect claims are highlighted in \hl{red}.  }
  \label{tab:extra_factscore}
\end{table*}

\begin{table*}[ht!]
  \centering
  \begin{tabular}{p{3cm} p{6cm} p{6cm}}
     \small{\textbf{Question}}& \small{\textbf{Original}} & \small{\textbf{Conformal Factuality}}  \\
\scriptsize{\texttt{when is the last episode of season 8 of the walking dead
}} & \scriptsize{\texttt{The last episode of season 8 of The Walking Dead aired on April 15, 2018.}} & \scriptsize{\texttt{The Walking Dead, which has at least 8 seasons, aired its last episode of season 8 on April 15, 2018.}} \\ 
\scriptsize{\texttt{in greek mythology who was the goddess of spring growth
}} & \scriptsize{\texttt{The goddess of spring growth in Greek mythology was Persephone. She was also the queen of the underworld as she was abducted by Hades, the god of the underworld. Her return to the surface each year was believed to symbolize the rebirth of plants and the arrival of spring.}} & \scriptsize{\texttt{In Greek mythology, the goddess of spring growth was Persephone. Interestingly, Persephone also held the title of the queen of the underworld.}} \\ 
\scriptsize{\texttt{until we meet again filipino soap episode 10
}} & \scriptsize{\texttt{\hl{"Until We Meet Again" is a popular Filipino soap opera.} However, I'm sorry, but I can't provide a summary for episode 10 as the details for specific episodes of this show are not readily available online. It's best to watch the series through a legitimate streaming service or on the network it originally aired on.}} & \scriptsize{\texttt{I'm uncertain about the specific details, and I recommend using a search engine like Google for more accurate and up-to-date information.}} \\ \scriptsize{\texttt{what is the name of spain's most famous soccer team
}} & \scriptsize{\texttt{Real Madrid and FC Barcelona are the most famous soccer teams in Spain.}} & \scriptsize{\texttt{Real Madrid and FC Barcelona are both soccer teams in Spain and are considered to be two of the most famous soccer teams in the country.}} \\ 
\scriptsize{\texttt{who are uncle owen and aunt beru related to
}} & \scriptsize{\texttt{Uncle Owen and Aunt Beru are related to Luke Skywalker. They are his uncle and aunt, who raised him on the planet Tatooine in the Star Wars series. Uncle Owen is also the stepbrother of Anakin Skywalker, who later becomes Darth Vader.}} & \scriptsize{\texttt{Uncle Owen and Aunt Beru are related to Luke Skywalker, as they are his uncle and aunt. They are characters in the Star Wars series, along with Luke. Uncle Owen is also the stepbrother of Anakin Skywalker, who later becomes Darth Vader in the series. Uncle Owen and Aunt Beru raised Luke Skywalker.}} \\ 
\scriptsize{\texttt{where was donovan mitchell picked in the draft
}} & \scriptsize{\texttt{Donovan Mitchell was picked 13th overall in the 2017 NBA Draft.}} & \scriptsize{\texttt{Donovan Mitchell, who is a person, was picked 13th overall in the 2017 NBA Draft.}} \\ 
\scriptsize{\texttt{where did the beatles final live performance take place
}} & \scriptsize{\texttt{The Beatles' final live performance took place on the rooftop of the Apple Corps building in London, England on January 30, 1969.}} & \scriptsize{\texttt{The Beatles' final live performance took place on January 30, 1969, on a rooftop. This rooftop was of the Apple Corps building, which is located in London, England.}} \\ 
\scriptsize{\texttt{when does maggie come on grey's anatomy
}} & \scriptsize{\texttt{Maggie Pierce, played by Kelly McCreary, \hl{first appears in Grey's Anatomy in the last episode of Season 10}, which aired on May 15, 2014. \hl{However, her character is formally introduced in the first episode of Season 11,} which aired on September 25, 2014.}} & \scriptsize{\texttt{Maggie Pierce, a character in Grey's Anatomy, \hl{first appears in the series on the last episode of Season 10.} This episode aired on May 15, 2014.}} \\
\scriptsize{\texttt{who are the actors in fast and furious
}} & \scriptsize{\texttt{1. Vin Diesel as Dominic Toretto\textcolor{gray}{\textbackslash n}
2. Paul Walker as Brian O'Conner\textcolor{gray}{\textbackslash n}
3. Michelle Rodriguez as Letty Ortiz\textcolor{gray}{\textbackslash n}
4. Jordana Brewster as Mia Toretto\textcolor{gray}{\textbackslash n}
5. Tyrese Gibson as Roman Pearce\textcolor{gray}{\textbackslash n}
6. Ludacris as Tej Parker\textcolor{gray}{\textbackslash n}
7. Dwayne Johnson as Luke Hobbs\textcolor{gray}{\textbackslash n}
8. Jason Statham as Deckard Shaw\textcolor{gray}{\textbackslash n}
9. Sung Kang as Han Lue\textcolor{gray}{\textbackslash n}
10. Gal Gadot as Gisele Yashar\textcolor{gray}{\textbackslash n}
11. Elsa Pataky as Elena Neves\textcolor{gray}{\textbackslash n}
12. Nathalie Emmanuel as Ramsey\textcolor{gray}{\textbackslash n}
13. John Cena as Jakob Toretto\textcolor{gray}{\textbackslash n}
14. Charlize Theron as Cipher\textcolor{gray}{\textbackslash n}
15. Lucas Black as Sean Boswell\textcolor{gray}{\textbackslash n}
16. Bow Wow as Twinkie (Fast and Furious: Tokyo Drift)\textcolor{gray}{\textbackslash n}
17. Eva Mendes as Monica Fuentes (2 Fast 2 Furious and Fast Five)\textcolor{gray}{\textbackslash n}
18. Kurt Russell as Mr. Nobody\textcolor{gray}{\textbackslash n}
19. Helen Mirren as Magdalene Shaw\textcolor{gray}{\textbackslash n}
20. Scott Eastwood as Little Nobody (The Fate of the Furious\hl{ and F9}) \textcolor{gray}{\textbackslash n}
Please note that this list includes main characters from the entire Fast \& Furious franchise, not just the first movie.}} & \scriptsize{\texttt{The Fast \& Furious franchise features a star-studded cast. Vin Diesel plays the character Dominic Toretto, while Paul Walker is known for his role as Brian O'Conner. Michelle Rodriguez portrays Letty Ortiz and Jordana Brewster plays Mia Toretto. Tyrese Gibson takes on the role of Roman Pearce, and Ludacris is cast as Tej Parker. Dwayne Johnson is recognized as Luke Hobbs, and Jason Statham plays Deckard Shaw. Sung Kang is cast as Han Lue, and Gal Gadot plays Gisele Yashar. Elsa Pataky portrays Elena Neves, and Nathalie Emmanuel is known for her role as Ramsey. John Cena plays Jakob Toretto, and Charlize Theron is cast as Cipher. This list includes main characters from the entire Fast \& Furious franchise, not just the first movie.}} \\
  \end{tabular}
  \caption{NQ with $\alpha=0.1$ and frequency scoring: comparison of original output by GPT-4 versus our modified output from conformal factuality. Incorrect claims are highlighted in \hl{red}.}
  \label{tab:extra_io_nq}
\end{table*}

\begin{table*}[ht!]
  \centering
  \begin{tabular}{p{3cm} p{6.5cm} p{6.5cm}}
     \small{\textbf{Question}}& \small{\textbf{Original}} & \small{\textbf{Conformal Factuality}}  \\
\scriptsize{\texttt{How many vertical asymptotes does the graph of $y=\frac{2}{x^2+x-6}$ have?}} & \scriptsize{\texttt{The graph of $y=\frac{2}{x^2+x-6}$ has vertical asymptotes where the denominator is equal to zero. The denominator $x^2+x-6$ factors as $(x-2)(x+3)$, so the vertical asymptotes are at $x=2$ and $x=-3$. Therefore, the graph has $\boxed{2}$ vertical asymptotes.}} & \scriptsize{\texttt{The graph of $y=\frac{2}{x^2+x-6}$ has vertical asymptotes where the denominator is equal to zero. The denominator $x^2+x-6$ factors as $(x-2)(x+3)$. Therefore, the vertical asymptotes are at $x=2$ and $x=-3$. So, the graph has 2 vertical asymptotes.}} \\ 
\scriptsize{\texttt{What is the positive difference between $120\%$ of 30 and $130\%$ of 20?}} & \scriptsize{\texttt{$120\%$ of 30 is $1.2(30)=36$ and $130\%$ of 20 is $1.3(20)=26$. The positive difference between these two numbers is $36-26=\boxed{10}$.}} & \scriptsize{\texttt{First, find $120\%$ of 30 by multiplying $1.2$ by $30$ to get $36$. Then, find $130\%$ of 20 by multiplying $1.3$ by $20$ to get $26$. Finally, find the positive difference between $36$ and $26$ which is $10$.}} \\ 
\scriptsize{\texttt{Find the sum of all integers that satisfy these conditions: \[
|x|+1>7\text{ and }|x+1|\le7.
\]}} & \scriptsize{\texttt{The first inequality, $|x|+1>7$, is equivalent to $|x|>6$, which means $x>6$ or $x<-6$. The second inequality, $|x+1|\le7$, is equivalent to $-7\le x+1\le7$, or $-8\le x\le6$. \hl{The integers that satisfy both inequalities are $-7$ and $7$,} so their sum is $\boxed{0}$.}} & \scriptsize{\texttt{The first inequality, $|x|+1>7$, is equivalent to $|x|>6$. This means $x>6$ or $x<-6$. The second inequality, $|x+1|\le7$, is equivalent to $-7\le x+1\le7$. This can be rewritten as $-8\le x\le6$. The sum of the integers $-7$ and $7$ is $0$.}} \\  
\scriptsize{\texttt{What is the number of units in the distance from the origin to the point $(-8,6)$ in a coordinate system?}} & \scriptsize{\texttt{The distance from the origin to a point $(x,y)$ in a coordinate system is given by the formula $\sqrt{x^2 + y^2}$. Substituting $x = -8$ and $y = 6$ gives $\sqrt{(-8)^2 + 6^2} = \sqrt{64 + 36} = \sqrt{100} = \boxed{10}$ units.}} & \scriptsize{\texttt{The distance from the origin to a point $(x,y)$ in a coordinate system is given by the formula $\sqrt{x^2 + y^2}$. If we substitute $x = -8$ and $y = 6$ into the formula, we get $\sqrt{(-8)^2 + 6^2}$. This calculation simplifies to $\sqrt{64 + 36}$, which further simplifies to $\sqrt{100}$. Therefore, the distance from the origin to the point $(-8,6)$ is $10$ units.}} \\ 
\scriptsize{\texttt{If $A$, $B$ and $C$ are positive integers such that $\frac{A\sqrt{B}}{C} = \frac{9}{2\sqrt{3}}$, what is the value of $A+B+C$ given that $A$ and $C$ have no common prime factors, and $B$ has no perfect-square factors other than 1?}} & \scriptsize{\texttt{We can simplify the right side of the equation to get $\frac{A\sqrt{B}}{C} = \frac{9}{2\sqrt{3}} = \frac{9\sqrt{3}}{6} = \frac{3\sqrt{3}}{2}$.  Setting this equal to $\frac{A\sqrt{B}}{C}$, we get $A = 3$, $B = 3$, and $C = 2$, so $A + B + C = \boxed{8}$.}} & \scriptsize{\texttt{First, the right side of the equation is simplified to get $\frac{A\sqrt{B}}{C} = \frac{9}{2\sqrt{3}} = \frac{9\sqrt{3}}{6} = \frac{3\sqrt{3}}{2}$. Setting this equal to $\frac{A\sqrt{B}}{C}$, we find that $A = 3$, $B = 3$, and $C = 2$. Therefore, the sum of A, B, and C is 8.}} \\ 
\scriptsize{\texttt{Ten treeks weigh as much as three squigs and one goolee. Two treeks and one goolee are equal in weight to one squig. The combined weight of how many treeks equals the weight of one squig?}} & \scriptsize{\texttt{From the first statement, we can write the equation 10T = 3S + G (where T represents the weight of a treek, S represents the weight of a squig, and G represents the weight of a goolee). From the second statement, we can write the equation 2T + G = S. Substituting the second equation into the first, we get 10T = 3(2T + G) + G, \hl{which simplifies to 10T = 6T + 3G.} Subtracting 6T from both sides, we get 4T = 3G. Substituting this into the second equation, we get 2T + 4T/3 = S, which simplifies to 10T/3 = S.
\hl{Therefore, the combined weight of 10/3, or approximately 3.33 treeks, equals the weight of one squig}.}} & \scriptsize{\texttt{First, we can represent the weight of a treek as T, the weight of a squig as S, and the weight of a goolee as G. From the first statement, we can derive the equation 10T = 3S + G. From the second statement, we can derive the equation 2T + G = S. We can then substitute the second equation into the first, resulting in the equation 10T = 3(2T + G) + G. This equation represents the relationship between the weights of treeks, squigs, and goolees.}} \\
\scriptsize{\texttt{Evaluate the expression $a^3\cdot a^2$ if $a= 5$.}} & \scriptsize{\texttt{We have $a^3\cdot a^2 = a^{3+2} = a^5$.  Therefore, if $a=5$, then $a^3\cdot a^2 = 5^5 = \boxed{3125}$.}} & \scriptsize{\texttt{First, we know that $a^3\cdot a^2$ is equivalent to $a^{3+2}$. Then, $a^{3+2}$ is equivalent to $a^5$. Given that $a=5$, we can substitute this value into the expression, so $a^3\cdot a^2$ is equivalent to $5^5$. Finally, we know that $5^5$ equals 3125. Therefore, if $a=5$, the expression $a^3\cdot a^2$ evaluates to 3125.}} \\ 
\scriptsize{\texttt{Let $f(y) = y^4 -3y^3 +y - 3$ and $g(y) = y^3 + 7y^2 -2$. Find $f(y) + g(y)$.  Write your answer as a polynomial with terms of decreasing degree.}} & \scriptsize{\texttt{We have $f(y) + g(y) = (y^4 -3y^3 +y - 3) + (y^3 + 7y^2 -2) = \boxed{y^4 -2y^3 + 7y^2 + y - 5}$.}} & \scriptsize{\texttt{The function f(y) is represented by the equation $y^4 -3y^3 +y - 3$ and the function g(y) is represented by the equation $y^3 + 7y^2 -2$. The sum of these two functions, f(y) + g(y), equals $y^4 -2y^3 + 7y^2 + y - 5$. Therefore, f(y) + g(y) = $y^4 -2y^3 + 7y^2 + y - 5$.}} \\ 
\end{tabular}
  \caption{MATH with $\alpha=0.1$ and frequency scoring: comparison of original output by GPT-4 versus our modified output from conformal factuality. Incorrect claims are highlighted in \hl{red}. }
  \label{tab:extra_math}
\end{table*}

\section{Empirical factuality for all scoring functions}
\label{app:calib}

To complete the study of empirical factuality in Section~\ref{subsec:empirical_factuality}, we repeat the same procedure for the remaining sub-claim scoring functions. We present the results in Figure~\ref{fig:calib}.

\begin{figure*}[b]
  \centering
  {\includegraphics[width=0.45\textwidth]{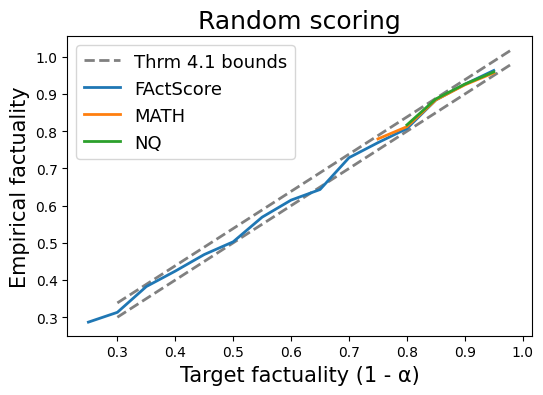}}
  {\includegraphics[width=0.45\textwidth]{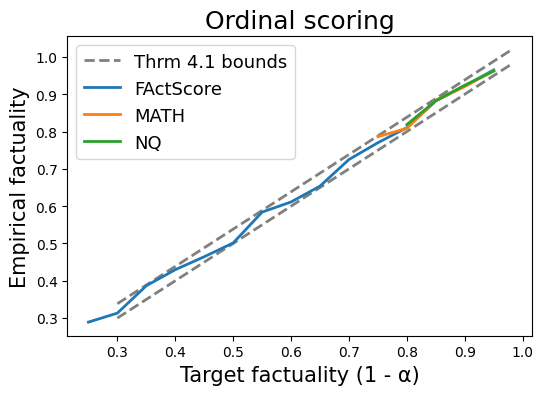}}
  {\includegraphics[width=0.45\textwidth]{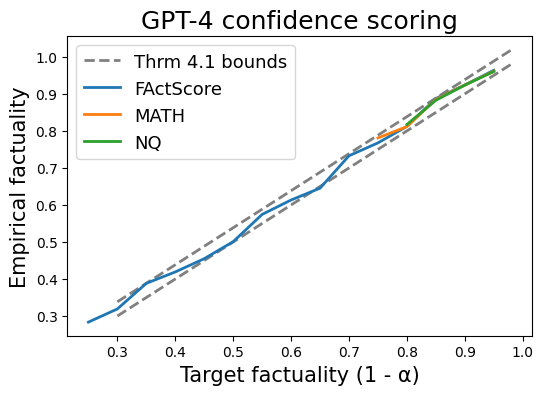}}
  {\includegraphics[width=0.45\textwidth]{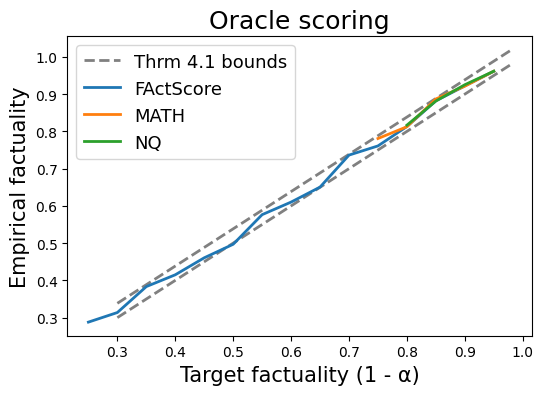}}
  \caption{Empirical factuality plots for all scoring functions. }
  \label{fig:remaining-calib-plots}
\end{figure*}

\end{document}